\documentclass[11pt]{article}
\usepackage{times}
\usepackage[margin=1in]{geometry}
\usepackage{color}

\usepackage{hyperref}
\usepackage{url}
\usepackage{graphicx}
\usepackage{authblk} % more modern

\usepackage{amsmath,amsthm,amssymb,epsfig,wrapfig,float}

\newtheorem{theorem}{Theorem}
\newtheorem{lemma}[theorem]{Lemma}
\newtheorem{corollary}[theorem]{Corollary}

\theoremstyle{definition}

\newcommand{\eps}{\varepsilon}

\newcommand{\mf}[1]{\noindent{\textcolor{blue}{\{{\bf MF:} #1\}}}}

\usepackage{natbib}

\usepackage{algorithm}
\usepackage{algorithmic,bm,multirow,array}
\newcommand{\cD}{\mathcal{D}}

\newcommand{\cA}{\mathcal{A}}
\newcommand{\cB}{\mathcal{B}}
\newcommand{\cT}{\mathcal{T}}

\newcommand{\Id}{\mathrm{I}}
\newcommand{\cF}{\mathcal{F}}

\newcommand{\ph}{h}

\renewcommand{\it}{\ell}

\newcommand{\bR}{\mathbb{R}}

\newcommand{\Tr}{\mathrm{Tr}}

\newcommand{\E}{\mathbb{E}}

\title{Global Convergence of Policy Gradient\\
Methods for the Linear Quadratic Regulator}

\author[1]{Maryam Fazel}
\author[2]{Rong Ge}
\author[1]{Sham M. Kakade}
\author[1]{Mehran Mesbahi}
\affil[1]{University of Washington, Seattle, WA, USA, \protect\\
 \url{mfazel@ee.washington.edu}, \url{sham@cs.washington.edu}, \url{mesbahi@aa.washington.edu}.  }
\affil[2]{Duke University, Durham, NC, USA, \protect\\\url{rongge@cs.duke.edu}.}

\begin{document}

\maketitle

\begin{abstract}
Direct policy gradient methods for
reinforcement learning and continuous control problems are a popular
approach for a variety of reasons: 1) they are easy to
implement without explicit knowledge of the underlying model, 2) they
are an ``end-to-end'' approach, directly optimizing the
performance metric of interest, 3) they inherently allow for
richly parameterized policies.
A notable drawback is that even in the most basic continuous control
problem (that of linear quadratic regulators), these methods must
solve a non-convex optimization problem, where little is understood
about their efficiency from both computational and statistical
perspectives. In contrast, system identification and model based planning
in optimal control theory have a much more solid theoretical footing,
where much is known with regards to their computational and
statistical properties.  This work bridges this gap showing that
(model free) policy gradient methods globally converge to the optimal
solution and are efficient (polynomially so in relevant problem
dependent quantities) with regards to their sample and computational
complexities.  
\end{abstract}

\section{Introduction}

Recent years have seen major advances in the control of uncertain
dynamical systems using reinforcement learning and data-driven
approaches; examples range from allowing robots to perform more
sophisticated controls tasks such as robotic hand
manipulation~\citep{Tassa12,Hertzmann,Vikash16,Levine16,Tobin17,RajeswaranKGSTL17},
to sequential decision making in game domains, e.g., 
AlphaGo~\citep{AlphaGo} and Atari game playing~\citep{Atari}. Deep
reinforcement learning (DeepRL) is becoming increasingly popular for
tackling such challenging sequential decision making problems.
%In robotics, these methods have shown increasing successes,
% including on popular benchmarks such as the OpenAI Gym-v1
% tasks~\citep{gym} and challenging hand-manipulation control problems~\cite{RajeswaranKGSTL17}. 

Many of these successes have relied on sampling based reinforcement
learning algorithms such as policy gradient methods, including the
DeepRL approaches. For these approaches, there is 
little theoretical understanding of their efficiency, either from a
statistical or a computational perspective.  In contrast, control
theory (optimal and adaptive control) has a rich body of tools, with
provable guarantees, for related sequential decision making problems,
particularly those that involve continuous control. These latter
techniques are often model-based---they estimate an explicit
dynamical model first (via system identification) and then design
optimal controllers.
% --- in contrast to the recent direct
%``model-free'' approaches, such as those in DeepRL (see ~\cite{rllab}
%for a review).

This work builds bridges between
these two lines of work, namely, between optimal control theory and
sample based reinforcement learning methods, using ideas from
mathematical optimization.
%This is done through the
%lens of mathematical optimization.
%In particular, we aim to
%better connect the model-free approaches in RL to the model-based
%approaches of system identification and control. 

\subsection{The optimal control problem}
%{\bf The setting. \/} 
In the standard optimal control problem, a dynamical system is described as
%the dynamics model $f_t$, where $f_t$ is specified as
\[
x_{t+1} = f_t(x_t,u_t,w_t) \, ,
\]
where $f_t$ maps a state $x_t\in \bR^d$, a control (the action) $u_t \in \bR^k$,
and a disturbance $w_t$, to the next state $x_{t+1}\in \bR^d$, starting from an initial state $x_0$. The
objective is to find the control input $u_t$ which minimizes the long
term cost, 
\begin{eqnarray*}
&\textrm{minimize} & \sum_{t=0}^T c_t(x_t,u_t)\\
&\textrm{such that} & x_{t+1} = f_t(x_t,u_t,w_t) \;\; t=0,\ldots,T.
\end{eqnarray*}
Here the $u_t$ are allowed to depend on the history of observed states, and $T$ is the time horizon (which can be finite or infinite).
In practice, this is often solved by considering the linearized control (sub-)problem where the dynamics are approximated by
\[
x_{t+1} = A_tx_t +B_tu_t +w_t,
\] 
and %$A_t$ and $B_t$ and 
the costs are approximated by a quadratic
function in $x_t$ and $u_t$, e.g.~\citep{Todorov04}.   The present paper 
considers an important special case: the time
homogenous, infinite horizon problem referred to as the linear quadratic regulator (LQR) problem. The results
herein can also be extended to the finite horizon, time inhomogenous
setting, discussed in Section~\ref{section:discussion}.
%which is discussed in the Appendix.  

We consider the following infinite horizon LQR problem,
\begin{eqnarray*}
&\textrm{minimize} & \E\left[\sum_{t=0}^\infty (x_t^\top Q x_t + u_t^\top R u_t )\right] \\
&\textrm{such that} & x_{t+1} = Ax_t + B u_t \, , \quad x_0\sim \cD \, ,
\end{eqnarray*}
where initial state $x_0\sim \cD$ is assumed to be randomly distributed
according to distribution $\cD$; the matrices
$A\in \bR^{d \times d}$ and $B\in \bR^{d \times k}$ are referred to as
system (or transition) matrices; $Q\in \bR^{d \times d}$ and
$R\in \bR^{k \times k}$ are %respectively, positive semidefinite and
%positive definite 
both positive definite matrices that 
%SK: we really do assume $P$ and $Q$ are positive not semidefinite 
parameterize the quadratic costs. For
clarity, this work does not consider a noise disturbance but only a
random initial state. The importance of (some) randomization for
analyzing direct methods is discussed in
Section~\ref{section:landscape}.

Throughout, assume that $A$ and $B$ are such that the optimal cost is finite (for example, the controllability of the pair $(A,B)$ would ensure this). Optimal control theory~\citep{Anderson:1990:OCL:79089,Evans2005,Bertsekas2011,Bertsekas:DP}
shows that the optimal control input can be written as a linear
function in the state, %, the optimal control policy has the form:
\[
u_t = -K^* x_t %\, , \textrm{where} \, \, \, K^{*} = -(B^{T} P B + R)^{-1} B^{T} P A.
\]
where $K^* \in \bR^{k\times d}$.  

{\bf Planning with a known model.\/} %With knowledge of the model,
For the infinite horizon LQR problem, planning can be achieved by solving the Algebraic
Riccati Equation (ARE), 
\begin{equation}\label{eq:ARE}
P= A^{T} P A + Q - A^{T} PB(B^{T} PB+R)^{-1} B^{T} P A \, ,
\end{equation}
for a positive definite matrix $P$ which parameterizes the 
``cost-to-go'' (the optimal cost from a state going
forward). 
The optimal control gain is then given as:
\begin{equation} \label{eq:K_from_P}
K^{*} = -(B^{T} P B + R)^{-1} B^{T} P A.
\end{equation}
To find $P$, there are iterative methods, algebraic solution methods,
and (convex) SDP formulations.
%
%%% brought up from appendix A
Solving the ARE is extensively studied; one
approach due to ~\citep{Kleinman1968} (for continuous time) and \citep{hewer} (for discrete time) is to simply run the recursion
$P_{k+1}= Q+ A^TP_k A - A^TP_k B (R+B^TP_k B)^{-1} B^TP_k A$ where 
$P_1=Q$, which converges to the unique positive semidefinite solution
of the ARE (since the fixed-point iteration is contractive). 
Other approaches are direct and are based on linear algebra, which carry out an
eigenvalue decomposition on a certain block matrix (called the Hamiltonian matrix) followed by a
matrix inversion \citep{lancaster1995algebraic}.
The LQR problem can also be expressed as a
semidefinite program (SDP) with variable $P$ as given in
\citep{SDP-LQR} (see Section \ref{section:solution_concepts}  in the supplement). 

%More broadly, even though there are convex formulations for planning, 
However, these formulations: 1) do not
directly parameterize the policy, 2) are not ``end-to-end''
approaches, in that they are not directly optimizing the cost function
of interest, and 3) it is not immediately clear how to utilize these
approaches in the model-free
setting, %such as in the standard reinforcement learning
where the agent only has simulation access. These issues are
outlined in Section~\ref{section:solution_concepts} of the supplement.
%(where we discuss how the SDP formulation \eqref{eq:SDP} is not a direct method that minimizes the cost over the set of feasible policies). 

%\sk{should probably briefly mention Youla here and then point to the
%  appendix for the longer discussion.}

\subsection{Contributions of this work}

Even in the most basic case of the standard linear quadratic regulator
model, little is understood as to how direct (model-free) policy
gradient methods fare. This work provides rigorous guarantees,
showing that, while in fact the approach deals with a non-convex problem, directly
using (model free) local search methods leads to finding the globally optimal
policy (i.e., a policy whose objective value is $\epsilon$-close to the optimal). The main contributions are as follows:

\begin{itemize}
\item (Exact case) Even with access to exact gradient %methods
  evaluation, 
  little is understood about whether or not convergence to the optimal
  policy occurs, even in the limit, due to the
  non-convexity of the problem.  This
  work shows that global convergence does indeed occur (and does so
  efficiently) for %local search based 
  gradient descent methods. 
  %This places these methods on a footing to be comparable to, say,
  %alternative iterative methods. \mf{last sentence not very clear to me...}
%The techniques
%  here borrow from mathematical optimization.
\item (Model free case) Without a model, this work shows how one can
  use simulated trajectories (as opposed to having knowledge of the
  model) in a stochastic policy gradient method, where provable
  convergence to a globally optimal policy is guaranteed, with (polynomially) efficient
  computational and sample complexities. 
  %The ideas borrow zero-th
  %order optimization and sample based policy gradient methods.
\item (The natural policy gradient) Natural policy gradient methods
  ~\citep{Kakade01} --- and related algorithms such as Trust Region
  Policy Optimization~\citep{Schulman15} and the natural actor
  critic~\citep{Peters} --- are some of the most widely used and
  effective policy gradient methods (see ~\cite{rllab}). While many
  results argue in favor of this method based on either information
  geometry~\citep{Kakade01,Bagnell:2003:CPS:1630659.1630805} or based
  on connections to actor-critic
  methods~\citep{Deisenroth:2013:SPS:2688186.2688187}, these results
  do not provably show an improved convergence rate.  This work is the
  first to provide a guarantee that the natural gradient
  method enjoys a considerably improved convergence rate over its
  naive gradient counterpart.
  %This work
  %provides a \emph{convergence rate} of the natural policy gradient,
 % which is in fact significantly better than that of naive gradient descent counterpart.
  % This is a plausible reason as to the improved efficiency of this
  % method.
\end{itemize}

More broadly, the techniques in this work merge ideas from optimal
control theory, mathematical optimization (first order and zeroth order), and
sample based reinforcement learning methods. These techniques may
ultimately help in improving upon the existing set of algorithms,
addressing issues such as variance reduction or improving upon the
natural policy gradient method (with, say, a Gauss-Newton method as in Theorem 7). The Discussion section touches upon some of
these issues.

\subsection{Related work}

In the reinforcement learning setting, the model is unknown, and the
agent must learn to act through its interactions with the environment. Here, solution concepts are
typically divided into: model-based approaches, where the agent
attempts to learn a model of the world, and model-free approaches,
where the agent directly learns to act and does not explicitly learn a
model of the world.  The related work on provably learning LQRs is reviewed
from this perspective.

\iffalse
In a nutshell, solving the infinite-horizon LQR problem, when the model is not known a priori, can be classified into model-based and model-free approaches, using the terminology from reinforcement learning. In model-based learning setting, the agent attempts to learn a model of the world, i.e., the plant, and then plan on using this model for control synthesis. In the model free setting, the agent does not explicitly learn a model of the world.

In a nutshell, solving the infinite-horizon LQR problem can be classified into model-based and model-free approaches. In the former case, the model of the ``plant," i.e., the system matrices $A$ and $B$, are assumed to be known, whereas in the latter case, they are assumed to be unknown for control synthesis. The model-free setup can further be refined as whether the approach relies on an explicit system identification, where the control synthesis involves the identification of the model prior to control design; as such this approach is really a "pseudo" model-free. A model-free approach can also bypass this explicit system identification step. 
 As our contribution is within a learning paradigm, we mainly refer to related work in model-free approaches for the LQR setup.
\fi
  
{\bf Model-based learning approaches. \/}\ In the context of LQRs, the
agent can attempt to learn the dynamics of ``the plant'' (i.e., the model)
and then plan, using this model, for control synthesis.  Here, the
classical approach is to learn the model with subspace-based
system identification~\citep{Ljung:1999:SIT:293154}. ~\citet{Fiechter:94}
provides a provable learning (and non-asymptotic) result, where the quality of the 
policy obtained is shown to be near optimal
(efficiency is in terms of the persistence of the training data and the
controllability Gramian). ~\citet{Abbasi-Yadkori2011} also provides
provable, non-asymptotic learning results in a regret context, using a bandit algorithm that achieves
lower sample complexity (by balancing exploration-exploitation more
effectively); the computational efficiency of this approach is less clear.

More recently, ~\citet{dean:2017} expands on an explicit system
identification process, where a robust control synthesis procedure is
adopted that relies on a coarse model of the plant matrices ($A$ and
$B$ are estimated up to some accuracy level, naturally leading to a
``robust control" setup to then design the controller based in the coarse model). Tighter analysis for sample complexity was given 
in \citet{tu2018least,simchowitz2018learning}. 
Arguably, this is the most general (and
non-asymptotic) result that is efficient from a statistical
perspective. Computationally, the method works with a finite horizon
to approximate the infinite horizon. This result only needs the plant
to be controllable; the work herein needs the stronger assumption that
the initial policy in the local search procedure is a stable
controller (an assumption which may be inherent to local search
procedures,  discussed in Section~\ref{section:discussion}). 
Another recent line of work \citep{hazan2017learning,hazan2018spectral,arora2018towards} treat the problem of learning a linear dynamical system as an online learning problem. 
\citep{hazan2017learning,arora2018towards} are restricted to systems with symmetric dynamics (symmetric $A$ matrix), 
while \citep{hazan2018spectral} handles a more general setting. This line of work can handle the case when there are latent states (i.e., when the observed output is a linear function of the state, and the state is not observed directly) and does not need to do system identification first.
On the other hand, they don't output a succinct linear policy as \citet{dean:2017} or this paper.

{\bf Model-free learning approaches. \/} Model-free approaches that do
not rely on an explicit system identification step typically either:
1) estimate value functions (or state-action values) through Monte
Carlo simulation which are then used in some approximate dynamic
programming variant~\citep{Bertsekas2011}, or 2) directly optimize a
(parameterized) policy, also through Monte Carlo simulation.
Model-free approaches for learning optimal controllers are not well
understood from a theoretical perspective.  Here, ~\citet{Bradtke1994}
provides an asymptotic learnability result using a value
function approach, namely $Q$-learning.

\section{Preliminaries and Background}
\subsection{Exact Gradient Descent}

This work seeks to characterize the behavior of (direct) policy
gradient methods, where  the policy is linearly parameterized, as
specified by a matrix $K \in
\bR^{k\times d}$ which generates the controls: 
\[
u_t = -K x_t
\]
for $t\geq 0$. The cost of this $K$ is denoted as:
\begin{eqnarray*}
C(K) := \E_{x_0\sim \cD}\left[ \sum_{t=0}^\infty (x_t^\top Q x_t + u_t^\top R u_t) \right]
\end{eqnarray*}
where $\{x_t,u_t\}$ is the trajectory induced by following $K$,
starting with $x_0\sim \cD$.  The importance of (some) randomization,
either in $x_0$ or noise through having a disturbance, for
analyzing gradient methods is discussed in
Section~\ref{section:landscape}.
%The optimization problem of interest is:
%\[
%\min_K C(K) \, .
%\]
Here,  $K^*$ is a minimizer of $C(\cdot)$.

%\subsection{Exact gradient descent}
Gradient descent on $C(K)$, with a fixed stepsize $\eta$,  follows the update rule:
\[
K \leftarrow K - \eta \nabla C(K)\, .
\]
%This work also considers the natural policy gradient
%method and a Gauss-Newton method.
It is helpful to explicitly write out the functional form of the gradient. Define $P_K$ as the solution to:
\[
P_K = Q + K^\top R K + (A-BK)^\top P_K  (A-BK) \, .
\]
and, under this definition, it follows that $C(K)$ can be written as:
\[
C(K) = \E_{x_0\sim \cD }  \, x_0^\top P_K x_0 \, .
\]
Also, define $\Sigma_K$ as the (un-normalized) state correlation matrix, i.e.
%\begin{equation}\label{eq:state_cov}
\[
\Sigma_K = \E_{x_0\sim \cD } \sum_{t=0}^\infty x_t x_t^\top \, .
\]
%\end{equation}
%which is the (un-normalized) state correlation matrix.

\begin{lemma}\label{lem:gradexpression}
(Policy Gradient Expression) 
%Define:
%\[
%E_K := (R+ B^\top P_K B) K - B^\top P_K A\, .
%\]
The policy gradient is:
\[
\nabla C(K) = 2 \left( (R+ B^\top P_K B) K - B^\top P_K A\right)
\Sigma_K 
%= 2 E_K \Sigma_K
\]

Later for simplicity, define $E_K$ to be
$$
E_K = \left( (R+ B^\top P_K B) K - B^\top P_K A\right),
$$
as a result the gradient can be written as $\nabla C(K) = 2E_K\Sigma_K$.
\end{lemma}

\begin{proof}
Observe:
\begin{align*}
C_K(x_0) & = x_0^\top P_K x_0\\
 & = x_0^\top\left(Q+K^\top R K\right) x_0\\
 &  \quad+
 x_0^\top (A-BK)^\top P_K (A-BK) x_0 \\&= x_0^\top\left(Q+K^\top R K\right) x_0\\
 &  \quad+ C_K((A-BK)x_0) \, .
\end{align*}
Let $\nabla$ denote the gradient with respect to $K$, note that $\nabla C_K((A-BK)x_0)$ has two terms (one with respect to $K$ in the subscript and one with respect to the input $(A-BK)x_0$), this implies 
\begin{eqnarray*}
\nabla C_K(x_0)%x_0^\top \nabla P_K x_0 
&=&
 2 R K x_0 x_0^\top  -2 B^\top P_K (A-BK) x_0 x_0^\top\\
 & &
+  \nabla C_K(x_1)|_{x_1=(A-BK)x_0}\\ %x_0^\top (A-BK)^\top\nabla P_K (A-BK) x_0 \\
%&=& 
%2 \left( (R+ B^\top P_K B) K - B^\top P_K A\right) x_0 x_0^\top  
%+ x_1^\top \nabla P_K x_1 \\
&=& 
2 \left( (R+ B^\top P_K B) K - B^\top P_K A\right) \sum_{t=0}^\infty x_t x_t^\top  
\end{eqnarray*}
using recursion and that $x_1 = (A-BK) x_0$. Taking expectations
completes the proof.
\end{proof}

\subsection{Review: (Model free) sample based policy gradient methods} \label{section:review_pg}

Sample based policy gradient methods introduce some
randomization for estimating the gradient.
%This section reviews some common stochastic approximation methods for estimating the
%gradient. These typically involve either considering
%stochastic policies or explicitly add in noise. %We review
%these methods in the context of linear quadratic control.

%\subsubsection*{REINFORCE}

{\bf REINFORCE.\/}\citep{Williams92,sutton2000policy}
Let $\pi_\theta(u|x)$ be a parametric stochastic policy, where $u \sim
\pi_\theta(\cdot|x)$. The policy gradient of the cost, $C(\theta)$, is:
\begin{align*}
&\nabla C(\theta) = \E \left[ \sum_{t=0}^\infty
Q_{\pi_{\theta_t}}(x_t,u_t)\nabla \log \pi_\theta(u_t|x_t)  \right],\\
&\textrm{where }
Q_{\pi_\theta }(x,u) = \E\left[\sum_{t=0}^\infty c_t | x_0=x, u_0=u\right] \, ,
\end{align*}
where the expectation is with respect to the trajectory $\{x_t,u_t\}$
induced under the policy $\pi_\theta$ and where $Q_{\pi_\theta }(x,u)$ is
referred to as the state-action value.
The REINFORCE algorithm uses Monte Carlo estimates of the gradient
obtained by simulating $\pi_\theta$.

\iffalse
The state-action value and advantage functions
are defined as follows:  
\[
Q_{\pi_\theta }(x,u) = \E[\sum_{t=0}^\infty = c_t | x_0=x, u_0=u] 
\]
where the expectation is with respect to the trajectory $\{x_t,u_t\}$
induced under the policy $\pi_\theta$, conditioning on $x$ and $u$ as
the initial state and control.
\fi

%\subsubsection*{The natural policy gradient}
{\bf The natural policy gradient.\/}
The natural policy
gradient~\citep{Kakade01} follows the update:
\begin{align*}
&\theta \leftarrow \theta - \eta \, G_\theta^{-1} \nabla C(\theta), \textrm{where:}\\
&
G_\theta = \E\left[\sum_{t=0}^\infty \nabla \log \pi_\theta(u_t|x_t)
  \nabla \log \pi_\theta(u_t|x_t)^\top \right] \, ,
\end{align*}
where $G_\theta$ is the Fisher information matrix. There are numerous
succesful related approaches \citep{Peters,Schulman15,rllab}.
%\[
%G_\theta = \E\left[\sum_{t=0}^\infty \nabla \log \pi_\theta(u_t|x_t) \nabla \log \pi_\theta(u_t|x_t)^\top \right]
%\]
An important special case is using a linear policy with additive Gaussian noise~\citep{RajeswaranLTK17}, i.e.
\begin{equation}\label {eq:gauss_policy}
\pi_K(x,u) = \mathcal{N}(Kx,\sigma^2 \Id)
\end{equation}
where $K \in \bR^{k\times d}$ and $\sigma^2$ is the noise variance. Here, the
natural policy gradient of $K$ (when $\sigma$ is considered fixed) takes the form:
\begin{equation}\label{eq:ng_def}
K \leftarrow K - \eta \nabla C(K) \Sigma_K^{-1}
\end{equation}
To see this, one can verify that the Fisher matrix of size $kd \times kd$, which is indexed as
$[G_K]_{(i,j),(i',j')}$ where $i,i' \in \{1,\ldots k\}$ and $j,j'\in \{1,\ldots d\}$,  has a block diagonal form where the
only non-zeros blocks are
$
[G_K]_{(i,\cdot),(i,\cdot)} = \Sigma_K
$
(this is the block corresponding to the $i$-th coordinate of the
action, as $i$ ranges from $1$ to $k$). This form holds more generally, for any
diagonal noise.

%\subsubsection*{Zeroth-order optimization (and evolutionary strategies)}
{\bf Zeroth order optimization.\/}
Zeroth order optimization is a generic
procedure~\citep{DFO_book,Nesterov2015} for optimizing a function $f(x)$,
using only query access to the function values of $f(\cdot)$ at input
points $x$ (and without explicit query access
to the gradients of $f$). This is also the approach in using ``evolutionary
strategies'' for reinforcement learning~\citep{Salimans17}. The generic
approach can be described as follows: define the perturbed function as
\[
f_{\sigma^2}(x) = \E_{\eps \sim \mathcal{N}(0,\sigma^2 \Id)}[f(x+\eps)]
%C_{\sigma^2}(K) = \E_{\eps\sim \mathcal{N}(0,\sigma^2 \Id)}[C(K+\eps)]
\] 
For small $\sigma$, the smooth function is a good approximation to the
original function. Due to the Gaussian smoothing, the gradient has the
particularly simple functional form (see \citet{DFO_book,Nesterov2015}): 
\[
\nabla f_{\sigma^2}(x) = \frac{1}{\sigma^2} \E_{\eps\sim
  \mathcal{N}(0,\sigma^2 \Id)}[f(x+\eps) \eps] \, .
%\nabla C_{\sigma^2}(K) = \frac{1}{\sigma^2} \E_{\eps\sim
%  \mathcal{N}(0,\sigma^2 \Id)}[C(K+\eps) \eps] \, .
\]
This expression implies a straightforward method to obtain an
unbiased estimate of the $\nabla f_{\sigma^2}(x)$, through obtaining
only the function values $f(x+\eps)$ for random $\eps$.

\iffalse
\subsection{Notation}
$\|Z\|$ denotes the spectral norm of a matrix $Z$; $\Tr(Z)$
denotes the trace of a square matrix; $\sigma_{\textrm{min}}(Z)$
denotes the minimal singular value of a square matrix $Z$.
\fi

\section{The (non-convex) Optimization Landscape} \label{section:landscape}
This section provides a brief characterization of the optimization
landscape, in order to help provide intuition as to why global
convergence is possible and as to where the analysis difficulties lie.  

\begin{lemma}\label{lem:main:nonconvex}
(Non-convexity) If $d\ge3$, there exists an LQR optimization problem, $\min_{K} C(K)$,  which is not
convex, quasi-convex, and star-convex.
\end{lemma}

The specific example is given in supplementary material (Section~\ref{section:non-convex_example}). In particular, there can be two matrices $K$ and $K'$ where both $C(K)$ and $C(K')$ are finite, but $C((K+K')/2)$ is infinite.

For a general non-convex optimization problem, gradient
descent may not even converge to the global optima in the limit. %; %furthermore,
%gradient descent may take an arbitrarily long time to even escape
%from a saddle point. Rong: not sure if this is relevant here, we
%later show that the objective function does not have any saddle
%point. Usually for a non-degenerate saddle gradient descent is fairly
%effective. 
The optimization problem of LQR satisfies a special {\em gradient domination} condition, which makes it much easier to optimize:

\begin{lemma}\label{lemma:domination:maintext}
(Gradient domination) Let $K^*$ be an optimal policy. Suppose $K$ has
finite cost and $\sigma_{\textrm{min}}(\Sigma_K) >0$. It holds that
%$C(\cdot)$ is gradient dominated in the following sense:
\begin{align*}
C(K)-C(K^*)
& %\leq 
%\|\Sigma_{K^*}\| \Tr(E_K^\top (R + B^\top P_KB)^{-1}%E_K)\\
%& 
%\leq
%\frac{\|\Sigma_{K^*}\|}{\sigma_{\textrm{min}}(R)} %\Tr(E_K^\top E_K)\\
%& 
\leq 
\frac{\|\Sigma_{K^*}\|}{ \sigma_{\textrm{min}}(\Sigma_K)^2\sigma_{\textrm{min}}(R)} \|\nabla C(K)\|_F^2.
\end{align*}
%For a lower bound, it holds that:
%\[
%C(K)-C(K^*)   \geq \frac{\mu}{\|R + B^\top P_KB\|}  \, \Tr(E_K^\top E_K) 
%\]
\end{lemma}

This lemma can be proved by analyzing the ``advantage'' of the optimal policy $\Sigma^*$ to $\Sigma$ in every step. The detailed lemma and the full proof is deferred to supplementary material.

%the following corollary (of
%Lemma~\ref{lemma:domination}) 
As a corollary, this lemma provides a characterization of the stationary points.

\begin{corollary}
(Stationary point characterization) If $\nabla C(K) = 0$, then either $K$ is an optimal policy or
$\Sigma_K$ is rank deficient.
\end{corollary}

Note that the covariance $\Sigma_K \succeq \Sigma_0 := \E_{x_0\sim \cD } x_0 x_0^\top$. Therefore, this lemma is the motivation for using a distribution over $x_0$
(as opposed to a deterministic starting point):
$\E_{x_0\sim \cD } x_0 x_0^\top$ being full rank guarantees that
$\Sigma_K$ is full rank, which implies all stationary points are a
global optima. An additive disturbance in the dynamics model also suffices.

The concept of gradient domination is important in the non-convex
optimization literature~\citep{Pol63,NesterovP06,Karimi2016}. A function $f: \bR^d \rightarrow \bR$ is said to be
gradient dominated if there exists some constant $\lambda$, such that for all $x$,
\[
f(x) - \min_{x'} f(x') \leq \lambda \|\nabla f(x)\|^2 \, .
\]
If a function is gradient dominated, this implies that if the magnitude of the gradient is
small at some $x$, then the function value at $x$ will be close to
that of the optimal function value.

Using the fact that $\Sigma_K \succeq \Sigma_0$, the following corollary of Lemma~\ref{lemma:domination:maintext} shows that
$C(K)$ is gradient dominated.

\begin{corollary}\label{cor:gradientdom:main}
(Gradient Domination) Suppose $\E_{x_0\sim \cD } x_0 x_0^\top$
is full rank. Then $C(K)$ is gradient dominated, i.e.
\[
C(K) - C(K^*) \leq \lambda \langle \nabla C(K),\nabla C(K)\rangle
\]
where  $\lambda = \frac{\|\Sigma_{K^*}\|}{ \sigma_{\textrm{min}}(\Sigma_0)^2\sigma_{\textrm{min}}(R)}$ is a problem dependent constant (and $\langle \cdot,\cdot\rangle$
denotes the trace inner product). 
\end{corollary}

Naively, one may hope that gradient domination immediately implies that
gradient descent converges quickly to the global optima. This would indeed be the case if
the $C(K)$ were a smooth function\footnote{A differentiable function $f(x)$ is said to
  be smooth if the gradients of $f$ are continuous. Equivalently, see
  the definition in Equation~\ref{eq:smoothness}.}: if it were the case that $C(K)$ is
both gradient dominated \emph{and} smooth, then classical mathematical optimization
results~\citep{Pol63} would not only immediately imply global
convergence, these results would also imply convergence at a linear
rate. These results are not immediately applicable due to it is not straightforward to characterize the
(local) smoothness properties of $C(K)$; this is a difficulty well
studied in the optimal control theory literature, related to
robustness and stability.

Similarly,  one may hope that recent results on escaping
saddle points~\citep{NesterovP06,GeHJY15,DBLP:conf/icml/Jin0NKJ17}
immediately imply that gradient descent converges quickly to the
global optima, due to that there are no (spurious)
local optima.  Again, for reasons related to smoothness this is not
the case.

The main reason that the LQR objective cannot satisfy the smoothness condition globally is that the objective becomes infinity when the matrix $A-BK$ becomes unstable (i.e. has an eigenvalue that is outside of the unit circle in the complex plane). At the boundary between stable and unstable policies, the objective function quickly becomes infinity, which violates the traditional smoothness conditions because smoothness conditions would imply quadratic upper-bounds for the objective function.

To solve this problem, it is observed that when the policy $K$ is not too close to the boundary, the objective satisfies an almost-smoothness condition:

\begin{lemma}\label{lemma:smoothness:maintext}
(``Almost'' smoothness) $C(K)$ satisfies:
\begin{align*}
&C(K')-C(K) 
 = 
-2\Tr(\Sigma_{K'}(K-K')^\top E_K)\\
& +
\Tr(\Sigma_{K'}(K-K')^\top (R + B^\top P_KB) (K-K'))
\end{align*}
\end{lemma}

To see why this is related to smoothness (e.g. compare to
Equation~\ref{eq:smoothness}), suppose $K'$ is
sufficiently close to $K$ so that:
\begin{equation*}
\Sigma_{K'} \approx \Sigma_K + O(\|K-K'\|)
\end{equation*}
and the leading order term $2\Tr(\Sigma_{K'}(K'-K)^\top E_K)$ would
then behave as $\Tr((K'-K)^\top \nabla C(K))$, and the remaining terms will be second order in $K-K'$.

Quantify the Taylor approximation $\Sigma_{K'} \approx \Sigma_K + O(\|K-K'\|)$ is one of the key steps in proving the convergence of policy gradient.

\section{Main Results}

%This section provides the global convergence results for model-based and
%model-free policy gradient methods.

First, results on exact gradient methods are provided. From an
analysis perspective, this is the natural starting point; once
global convergence is established for exact methods, the 
question of using simulation-based, model-free
methods can be approached with zeroth-order optimization methods (where gradients are not available, and can only be approximated using samples of the function value). 

{\bf Notation.\/}
$\|Z\|$ denotes the spectral norm of a matrix $Z$; $\Tr(Z)$
denotes the trace of a square matrix; $\sigma_{\textrm{min}}(Z)$
denotes the minimal singular value of a square matrix $Z$.
Also, it is helpful to define
\[
\mu := \sigma_{\textrm{min}}(\E_{x_0\sim \cD } x_0 x_0^\top)
\]

\subsection{Model-based optimization: exact gradient methods}

We consider three exact update rules. For gradient descent, the update is
\begin{equation}
 K_{n+1} =  K_n - \eta \nabla C(K_n). \label{eq:exact_gd}
\end{equation}
For natural policy gradient descent, the direction is defined so that it is
consistent with the stochastic case, as per
Equation~\ref{eq:ng_def}, in the exact case the update is:
\begin{equation}
 K_{n+1} =  K_n - \eta \nabla C(K_n) \Sigma_{K_n}^{-1} \label{eq:exact_ngd}
 \end{equation}
For Gauss-Newton method, the update is:
\begin{equation}
 K_{n+1} =  K_n - 
\eta (R+ B^\top P_{K_n} B)^{-1}  \nabla C(K_n) \Sigma_{K_n}^{-1} \, .\label{eq:exact_gn}
\end{equation}
%The natural policy gradient descent  
The standard policy iteration algorithm\citep{howard1964dynamic} that tries to optimize a one-step deviation from the current policy is equivalent to a special case of the Gauss-Newton method when $\eta = 1$ %It is straightforward to verify that the
%policy iteration algorithm is a special case of the Gauss-Newton
%method when $\eta=1$ 
(for the case of policy iteration, convergence in
the limit is provided in ~\cite{Todorov04,Ng2002AGC,Liao91}, along with
local convergence rates.)

The Gauss-Newton method requires the most complex oracle to implement: it requires
access to $\nabla C(K)$, $\Sigma_K$, and $R+ B^\top P_K B$; it also
enjoys the strongest convergence rate guarantee. At the
other extreme, gradient descent requires oracle access to
only $\nabla C(K)$ and has the slowest convergence rate. The natural policy gradient sits in
between, requiring oracle access to $\nabla C(K)$ and
$\Sigma_K$, and having a convergence rate between the other two methods.

\iffalse
The theorem provides convergences rates for all three cases, starting with the most complex method, Gauss Newton,
which also enjoys the strongest convergence guarantee, and ending with
the simplest method, gradient descent, which has the weakest convergence rate (though it is implementable with oracle access to only the gradients.).
\fi

\begin{theorem}\label{theorem:gd_exact}
(Global Convergence of Gradient Methods) Suppose $C(K_0)$ is
finite and $\mu>0$.
\begin{itemize}
\item Gauss-Newton case: For a stepsize  $\eta=1$ and for
\[
N \geq \frac{\|\Sigma_{K^*}\|}{\mu} \, \log \frac{C(K_0) -C(K^*)}{\eps} \, ,
\]
the Gauss-Newton algorithm (Equation~\ref{eq:exact_gn}) enjoys the following performance bound:
\[
C(K_N) -C(K^*) \leq \eps
\]
\item Natural policy gradient case: For a stepsize 
\[
\eta = \frac{1}{\|R\| + \frac{\|B\|^2 C(K_0)}{\mu}}
\]
and for
\begin{align*}
N \geq &\frac{\|\Sigma_{K^*}\|}{\mu} \,
\left(\frac{\|R\|}{\sigma_{\textrm{min}}(R)} + 
\frac{\|B\|^2 C(K_0)}{\mu \sigma_{\textrm{min}}(R)} \right) 
\, \\
&\log \frac{C(K_0) -C(K^*)}{\eps} \, ,
\end{align*}
natural policy gradient descent (Equation~\ref{eq:exact_ngd}) enjoys the
following performance bound:
\[
C(K_N) -C(K^*) \leq \eps \, .
\]
\item Gradient descent case: For an appropriate (constant) setting of the stepsize $\eta$,
\[
\eta = \mathrm{poly}\left(\frac{\mu\sigma_{\textrm{min}}(Q)}{C(K_0)},\frac{1}{\|A\|},\frac{1}{\|B\|},\frac{1}{\|R\|},
\sigma_{\textrm{min}}(R) \right)
\]
and for
\begin{align*}
N \geq &\frac{\|\Sigma_{K^*}\|}{\mu} \log
\frac{C(K_0) -C(K^*)}{\eps} \, \\
& \,\mathrm{poly}\left(\frac{C(K_0)}{\mu\sigma_{\textrm{min}}(Q)},\|A\|,\|B\|,\|R\|,
\frac{1}{\sigma_{\textrm{min}}(R)} \right)
\, ,
\end{align*}
gradient descent (Equation~\ref{eq:exact_gd}) enjoys the
following performance bound:
\[
C(K_N) -C(K^*) \leq \eps\, .
\]
\end{itemize}
\end{theorem}

In comparison to model-based approaches, these results require the (possibly) stronger
assumption that the initial policy is a
stable controller, i.e. $C(K_0)$ is finite (an assumption which may be inherent to local search
procedures). The Discussion mentions this as direction of future work. 

The proof for Gauss-Newton algorithm is simple based on the characterizations in Lemma~\ref{lemma:domination:maintext} and Lemma~\ref{lemma:smoothness:maintext}, and is given below. The proof for natural policy gradient and gradient descent are more involved, and are deferred to supplementary material.

\begin{lemma}\label{lem:gaussnewton:main}
Suppose that:
\[
K' = K - \eta (R+ B^\top P_K B)^{-1}  \nabla C(K) \Sigma_K^{-1}\, , .
\]
If $\eta \leq 1$, then
\[
C(K') -C(K^*) \leq \left(1- \frac{\eta \mu }{\|\Sigma_{K^*}\|}\right) (C(K)-C(K^*))
\]
\end{lemma}

\begin{proof}
Observe $K' = K - \eta (R+ B^\top P_K B)^{-1} E_K$.
Using Lemma~\ref{lemma:smoothness:maintext} and the condition on $\eta$,
\begin{eqnarray*}
& & C(K')-C(K) \\
&=&
-2\eta\Tr(\Sigma_{K'} E_K^\top (R+ B^\top P_K B)^{-1} E_K) + \\
& &\eta^2  \Tr(\Sigma_{K'} E_K^\top (R+ B^\top P_K B)^{-1} E_K)\\
&\leq&-\eta \Tr(\Sigma_{K'} E_K^\top (R+ B^\top P_K B)^{-1}
    E_K)\\
&\leq&-\eta \sigma_{\textrm{min}}(\Sigma_{K'})\Tr(E_K^\top (R+ B^\top P_K B)^{-1}
    E_K)\\
& \leq & - \eta \mu\Tr(E_K^\top (R+ B^\top P_K B)^{-1} E_K) \\
& \leq & - \eta \frac{\mu}{\|\Sigma_{K^*}\|} (C(K)-C(K^*)) \, ,
\end{eqnarray*}
where the last step uses Lemma~\ref{lemma:domination:maintext}.
\end{proof}

With this lemma, the proof of the convergence rate of the Gauss Newton
algorithm is immediate.

\begin{proof} (of Theorem~\ref{theorem:gd_exact}, Gauss-Newton case)
The theorem is due to that $\eta=1$ leads to a contraction of  $1- \frac{\eta \mu
}{\|\Sigma_{K^*}\|}$ at every step.
\end{proof}

\subsection{Model free optimization: sample based policy gradient methods}

\begin{algorithm}[t]
	\caption{ Model-Free Policy Gradient (and Natural Policy
          Gradient) Estimation}
	\label{algo}
	\begin{algorithmic}[1]
		\STATE Input: $K$, number of trajectories $m$, roll out length $\ell$,
                smoothing parameter $r$, dimension $d$
%		\INPUT stuff
		\FOR{$i = 1, \cdots m$}
		\STATE Sample a policy $\widehat K_i = K+U_i$, where $U_i$ is
                drawn uniformly at random over matrices whose (Frobenius) norm is $r$.
		\STATE Simulate $\widehat K_i$ for $\ell$ steps starting
                from $x_0\sim \cD$. Let $\widehat C_i$ and $\widehat \Sigma_i$ be
                the empirical estimates:
\[
\widehat C_i = \sum_{t=1}^\ell c_t \, , \quad \widehat \Sigma_i = \sum_{t=1}^\ell x_t x_t^\top
\]
where $c_t$ and $x_t$ are the costs and states on this trajectory.
%		\STATE Update:
		\ENDFOR
		\STATE Return the (biased) estimates:
\[
\widehat{\nabla C(K)} = \frac{1}{m} \sum_{i=1}^m \frac{d}{r^2} \widehat C_i U_i
\, , \quad
\widehat{\Sigma_K} = \frac{1}{m} \sum_{i=1}^m \widehat \Sigma_i 
\]
	\end{algorithmic}
\end{algorithm}

In the model free setting, the controller has only simulation access
to the model; the model parameters, $A$, $B$, $Q$ and $R$, are unknown. 
The standard optimal control theory approach is to use system
identification to learn the model, and then plan with this learned
model 
%(see~\citet{dean:2017} for a rigorous characterization of this
%approach). 
This section
proves that model-free, policy gradient methods also lead to globally
optimal policies, with both 
polynomial computational and sample complexities (in the relevant quantities).

Using a zeroth-order optimization approach (see
Section~\ref{section:review_pg}), Algorithm~\ref{algo} provides a
procedure to find (bounded bias) estimates, $\widehat{\nabla
  C(K)}$ and $\widehat \Sigma_K$, of both $\nabla C(K)$ and
$\Sigma_K$. These can then be used in the policy gradient
and natural policy gradient updates. For policy gradient we have
\begin{equation}
 K_{n+1} =  K_n - \eta \widehat{\nabla C(K_n)}. \label{eq:approx_gd}
 \end{equation}
 For natural policy gradient we have:
\begin{equation}
 K_{n+1} =  K_n - \eta \widehat{\nabla C(K_n)} \widehat{\Sigma}_{K_n}^{-1}
             \, .\label{eq:approx_ngd}
\end{equation}
In both Equations \eqref{eq:approx_gd} and \eqref{eq:approx_ngd}, Algorithm~\ref{algo} is called at every iteration to
provide the estimates of  $\nabla C(K_n)$ and $\Sigma_{K_n}$.

The choice of using zeroth order optimization vs using REINFORCE (with
Gaussian additive noise, as in Equation~\ref{eq:gauss_policy}) is
primarily for technical reasons\footnote{The correlations in the
  state-action value estimates in REINFORCE are more challenging to
  analyze.}.  It is plausible that the REINFORCE estimation procedure
has lower variance. One additional minor difference, again for
technical reasons, is that Algorithm~\ref{algo} uses a perturbation
from the surface of a sphere (as opposed to a Gaussian perturbation).

\begin{theorem}\label{theorem:gd_approx}
(Global Convergence in the Model Free Setting) Suppose $C(K_0)$ is
finite, $\mu>0$, and that $x_0\sim \cD$ has norm 
bounded by $L$ almost surely. %The notation $h$ represents fixed polynomials in $C(K_0),\frac{1}{\mu}, \frac{1}{\sigma_{\textrm{min}}(Q)},\|A\|,\|B\|,\|R\|,
%\frac{1}{\sigma_{\textrm{min}}(R)}$ and possibly $d,1/\epsilon$,
%$L^2/\mu$ as specified in parameters.
Also, for both the policy gradient method and the natural policy
gradient method,  suppose Algorithm~\ref{algo} is called with parameters:
\begin{align*}
m, \it, 1/r = &\textrm{poly}\left(C(K_0),\frac{1}{\mu}, \frac{1}{\sigma_{\textrm{min}}(Q)},\|A\|,\|B\|,\|R\|,\right.\\
&\left.\frac{1}{\sigma_{\textrm{min}}(R)}, d,1/\epsilon, L^2/\mu\right)\, .
\end{align*}
\begin{itemize}
\item Natural policy gradient case: For a stepsize 
\[
\eta = \frac{1}{\|R\| + \frac{\|B\|^2 C(K_0)}{\mu}}
\]
and for
\begin{align*}
N\geq &\frac{\|\Sigma_{K^*}\|}{\mu} \,
\left(\frac{\|R\|}{\sigma_{\textrm{min}}(R)} + 
\frac{\|B\|^2 C(K_0)}{\mu \sigma_{\textrm{min}}(R)} \right) 
\,\\ & \log \frac{2(C(K_0) -C(K^*))}{\eps} \, ,
\end{align*}
%if the gradient and variance are estimated using Algorithm~\ref{algo} with 
%$$n, \it, 1/r = \textrm{poly}\left(C(K_0),\frac{1}{\mu}, \frac{1}{\sigma_{\textrm{min}}(Q)},\|A\|,\|B\|,\|R\|,
%\frac{1}{\sigma_{\textrm{min}}(R)}, d,1/\epsilon, L^2/\mu\right),$$ 
%the variance $\Sigma_K$ is estimated using Equation \eqref{eq:varsampleestimate} with same number samples and iterations,
then, with high probability, i.e. with probability greater than
$1-\exp(-d)$, the natural policy gradient descent update
(Equation~\ref{eq:approx_ngd}) enjoys the following performance bound:
\[
C(K_N) -C(K^*) \leq \eps \, .
\]
\item Gradient descent case: For an appropriate (constant) setting of the stepsize $\eta$,
\[
\eta = \mathrm{poly}\left(\frac{\mu\sigma_{\textrm{min}}(Q)}{C(K_0)},\frac{1}{\|A\|},\frac{1}{\|B\|},\frac{1}{\|R\|},
\sigma_{\textrm{min}}(R) \right)
\]
and for
\begin{align*}
N \geq &\frac{\|\Sigma_{K^*}\|}{\mu} \log
\frac{C(K_0) -C(K^*)}{\eps} \, \\
&\times \,\mathrm{poly}\left(\frac{C(K_0)}{\mu\sigma_{\textrm{min}}(Q)},\|A\|,\|B\|,\|R\|,
\frac{1}{\sigma_{\textrm{min}}(R)} \right)
\, ,
\end{align*}
%the gradient is estimated using Algorithm~\ref{algo}
%with $$n, \it, 1/r = \textrm{poly}\left(C(K_0),\frac{1}{\mu}, \frac{1}{\sigma_{\textrm{min}}(Q)},\|A\|,\|B\|,\|R\|,
%\frac{1}{\sigma_{\textrm{min}}(R)}, d,1/\epsilon, L^2/\mu\right),$$ 
then, with high probability,
gradient descent (Equation~\ref{eq:approx_gd}) enjoys the
following performance bound:
\[
C(K_N) -C(K^*) \leq \eps\, .
\]
\end{itemize}
\end{theorem}

This theorem gives the first polynomial time guarantee for policy gradient and natural policy gradient algorithms in the LQR problem. 

\paragraph{Proof Sketch}

The model free results (Theorem~\ref{theorem:gd_approx}) are proved in the following three steps:

\begin{enumerate}
\item Prove that when the roll out length $\ell$ is large enough, the cost function $C$ and the covariance $\Sigma$ are approximately equal to the corresponding quantities at infinite steps.
\item Show that with enough samples, Algorithm~\ref{algo} can estimate both the gradient and covariance matrix within the desired accuracy.
\item Prove that both gradient descent and natural gradient descent can converge with a similar rate, even if the gradient/natural gradient estimates have some bounded perturbations.
\end{enumerate}

The proofs are technical and are deferred to supplementary material. We have focused on proving polynomial relationships in our complexity bounds, and did not optimize for the best dependence on the relevant parameters.

\section{Conclusions and Discussion}\label{section:discussion}

This work has provided provable guarantees that model-based gradient methods and
model-free (sample based) policy gradient methods convergence to the
globally optimal solution, with finite polynomial computational and sample complexities.
Taken together, the results herein place these popular and practical policy gradient
approaches on a firm theoretical footing, making them comparable to
other principled approaches (e.g., subspace system identification methods and algebraic iterative approaches).

{\bf Finite $C(K_0)$ assumption, noisy case, and finite horizon case.}
These methods allow for extensions to the noisy
case and the finite horizon case.  This work also made the assumption
that $C(K_0)$ is finite, which may not be easy to achieve in some
infinite horizon problems. The simplest way to address this is to
model the infinite horizon problem with a finite horizon
one; the techniques developed in
Section~\ref{sec:finitehorizonsimulation} shows this is
possible. This is an important direction for future work.

{\bf Open Problems.}
\begin{itemize}
\item Variance reduction: This work only proved efficiency from a polynomial sample size
perspective. An interesting future direction would be in how to
rigorously combine variance reduction methods and model-based methods to further decrease the
sample size.
\item A sample based Gauss-Newton approach: This work showed how the
  Gauss-Newton algorithm improves over even the natural policy
  gradient method, in the exact case.  A practically relevant question
  for the Gauss-Newton method would be how to
  both: a) construct a sample based estimator b) extend this scheme to
  deal with (non-linear) parametric policies.
\item Robust control: In model based approaches, optimal control
  theory provides efficient procedures to deal with (bounded) model
  mis-specification. An important question is how to provably
  understand robustness in a model free setting.

\end{itemize}

\subsection*{Acknowledgments}
% used initials to make this paragraph a bit shorter, feel free to change if full name is preferred.
Support from DARPA Lagrange Grant FA8650-18-2-7836 (to M.\ F., M.\ M., and S.\ K.) and 
from ONR award N00014-12-1-1002 (to M.\ F.\ and M.\ M.) is gratefully
acknowledged.
S. K. gratefully acknowledges funding from the Washington Research
Foundation for Innovation in Data-intensive Discover  and the ONR award N00014-18-1-2247. 
S.\ K.\ thanks Emo Todorov, Aravind Rajeswaran, Kendall Lowrey, Sanjeev
Arora, and Elad Hazan for helpful discussions.  
S.\ K.\ and M.\ F.\ also thank Ben Recht for helpful discussions. 
R.\ G.\ acknowledges funding from NSF CCF-1704656. 
We thank Jingjing Bu from University of Washington for running the numerical simulations in Section~\ref{sec:exp} in supplementary material. We also thank Bin Hu for reading the paper carefully and pointing out a missing step in the proof, which is now addressed in Section~\ref{sec:exactanalysis}.

\bibliography{references,learning-control}
\bibliographystyle{plainnat}

\appendix

\section{Planning with a model}\label{section:solution_concepts}

This section briefly reviews some parameterizations and solution methods for the classic LQR and related problems from control theory. 

{\bf Finite horizon LQR.} First, consider the finite horizon case. The
basic approach is to view it as a dynamic program with the value
function $x_t^T P_t x_t$, where
\begin{equation*}%\label{eq:r_recursion} 
P_{t-1}= Q+ A^TP_t A -
A^TP_{t} B (R+B^TP_{t} B)^{-1} B^TP_{t} A, 
\end{equation*} 
which in turn gives the optimal control 
\[
u_t= -K_t x_t= -(R+B^T P_{t+1}B)^{-1} B^T
P_{t+1} A x_t,
\]
(recursions run backward in time).

Another approach is to view the LQR problem as a linearly-constrained
Quadratic Program in all $x_t$ and $u_t$ (where the constraints are
given by the dynamics, and the problem size equals the horizon). The
QP is clearly a convex problem, but this observation is not useful by
itself as the problem size grows with the horizon, and naive use of
quadratic programming scales badly. However, the special structure due
to the linearity of the dynamics allows for simplifications and a control-theoretic
interpretation as follows: the Lagrange multipliers in the QP can be interpreted
as ``co-state" variables, and they follow a recursion that runs
backwards in time known as the ``adjoint system" dynamics. Using Lagrange
duality, one can show that this approach is equivalent to solving the
Riccati recursion mentioned above.

Popular use of the LQR in control practice is often in the receding
horizon LQR, \cite{camacho2004model,rawlings2009model}: at time $t$,
an input sequence is found that minimizes the $T$-step ahead LQR cost
starting at the current time, then only the first input in the
sequence is used. The resulting static feedback gain converges to the
infinite horizon optimal solution as horizon $T$ becomes longer.

{\bf Infinite horizon LQR.} 
Here, the constrained optimization view
(QP) is not informative as the problem is infinite dimensional;
however the dynamic programming viewpoint readily extends. Suppose the system
$A,~B$ is controllable (which guarantees the optimal cost is finite). It
turns out that the value function and the optimal controller are
static (i.e., do not depend on $t$) and can be found by solving the
Algebraic Riccati Equation (ARE) given in \eqref{eq:ARE}. The optimal
$K$ can then be found from equation \eqref{eq:K_from_P}.

The main computational step is solving the ARE, which is
extensively studied (e.g. \citep{lancaster1995algebraic}).  One
approach due to ~\citep{Kleinman1968} (for continuous time) and \citep{hewer} (for discrete time) is to simply run the recursion
$P_{k+1}= Q+ A^TP_k A - A^TP_k B (R+B^TP_k B)^{-1} B^TP_k A$ where 
$P_1=Q$, which converges to the unique positive semidefinite solution
of the ARE (since the fixed-point iteration is contractive). 
Other approaches are direct and based on linear algebra, which carry out an
eigenvalue decomposition on a certain block matrix (called the Hamiltonian matrix) followed by a
matrix inversion \citep{lancaster1995algebraic}.

Direct computation of the control input has also been considered in
the optimal control literature, e.g., gradient updates in function
spaces~\citep{Polak1973}.  For the linear quadratic setup, direct
iterative computation of the feedback gain has been examined
in~\citep{Martensson2009}, and explored further
in~\citep{maartensson2012gradient} with a view towards distributed
implementations.  There methods are presented as local search
heuristics without provable guarantees of reaching the optimal policy.

{\bf SDP formulation.} The LQR problem can also be expressed as a
semidefinite program (SDP) with variable $P$, as given in
\citep{SDP-LQR} (section 5, equation (34), this is for a
continuous-time system but there are similar discrete-time
versions). This SDP can be derived by relaxing the equality in the
Riccati equation to an inequality, then using the Schur complement
lemma to rewrite the resulting Riccati inequality as linear matrix
inequality. The objective in the case of LQR is the trace of the
positive definite matrix variable, and the optimization problem (for the continuous time system) 
is given as
\begin{equation}\label{eq:SDP}
\begin{array}{ll}
\mbox{maximize} & x_0^T P x_0\\
\mbox{subject to} & 
   \left[ \begin{array}{ll} A^T P+PA+Q &  PB\\ B^T P & I \end{array}  \right]  \geq 0, \;\;\;\; P\geq 0,
\end{array}
\end{equation}
where the optimization variable is $P$ 
This SDP and its dual, and system-theoretic interpretations of its optimality conditions, have
been explored in~\citep{SDP-LQR}.
Note that while the optimal solution $P^*$ of this SDP is
the unique positive semidefinite solution to the Riccati equation,
which in turn gives the optimal policy $K^*$, other feasible $P$ (not
equal to $P^*$) do not necessarily correspond to a feasible, 
stabilizing policy $K$. This means that the feasible set of this SDP
is not a convex characterization of all $P$ that correspond to
stabilizing $K$. Thus it also implies that if one uses any
optimization algorithm that maintains iterates in the feasible set
(e.g., interior point methods), no useful policy can be extracted from
the iterates before convergence to $P^*$. 
For this reason, this convex formulation is not helpful for parametrizing the space of policies $K$
in a manner that supports the use of local search methods (those that
directly lower the cost function of interest as a function of policy $K$), which is the focus of
this work.

\section{Non-convexity of the set of stabilizing State Feedback Gains}
\label{section:non-convex_example}
In this section we prove Lemma~\ref{lem:main:nonconvex}. 
Let ${\cal K}(A,B)$ denote the set of state feedback gains $K$ such that $A-BK$ is stable, i.e., its eigenvalues are inside the unit circle in the complex plane. This set is generally nonconvex. A concise counterexample to convexity is provided here. Let $A$ and $B$ be $3 \times 3$ identity matrices and
\[
K_1=\left[ \begin{array}{ccc} 1  & 0 & -10\\  -1 &  1 &  0\\ 0 &  0 & 1 \end{array} \right] \quad \mbox{and} \quad K_2=\left[ \begin{array}{ccc} 1  & -10 & 0\\  0 &  1 &  0\\ -1 & 0 & 1 \end{array} \right].
\]
Then the spectra of $A-BK_1$ and $A-BK_2$ are both concentrated at the origin, yet two of the eigenvalues of $A-B\widehat{K}$ with $\widehat{K}=(K_1+K_2)/2$, are outside of the unit circle in the complex plane.

\section{Analysis: the exact case}
\label{sec:exactanalysis}
This section provides the analysis of the convergence rates of the
(exact) gradient based methods. First, some helpful lemmas for the
analysis are provided.

Throughout, it is convenient to use the following definition:
\[
E_K := (R+ B^\top P_K B) K - B^\top P_K A\, .
\]
The policy gradient can then be written as:
\[
\nabla C(K) = 2 \left( (R+ B^\top P_K B) K - B^\top P_K A\right)
\Sigma_K = 2 E_K \Sigma_K
\]

\subsection{Helper lemmas}

Define the value $V_K(x)$, the state-action value $Q_K(x,u)$, and the
advantage $A_K(x,u)$. $V_K(x,t)$ is the cost of the policy
starting with $x_0=x$ and proceeding with $K$ onwards:
\begin{eqnarray*}
V_K(x) & := & 
\sum_{t=0}^\infty \left(x_t^\top Q x_t + u_t^\top R u_t \right) \\
& = & x^\top P_K x \, .
\end{eqnarray*}
$Q_K(x,u)$ is the cost of the policy 
starting with $x_0=x$, taking action $u_0=u$ and then proceeding with $K$ onwards:
\[
%Q_K(x,u) := x^\top Q x + u^\top R u + \sum_{t=1}^\infty\left(x_t^\top Q x_t + u_t^\top R u_t  \right)\,
Q_K(x,u) := x^\top Q x + u^\top R u + V_K(Ax+Bu) \,
\] 
The advantage $A_K(x,u)$ is:
\[ 
A_K(x,u) = Q_K(x,u) - V_K(x) \, .
\]
The advantage can be viewed as the change in cost starting at state
$x$ and taking a one step deviation from the policy $K$.

The next lemma is identical to that in \citep{Kakade02,KakadeThesis} for Markov decision processes.

\begin{lemma}\label{lemma:helper}
(Cost difference lemma) Suppose $K$ and $K'$ have finite costs. Let $\{x'_t\}$ and $\{u'_t\}$ be state and
action sequences generated by $K'$, i.e. starting with $x'_0=x$ and
using $u'_t = -K' x'_t$. It holds that:
\[
V_{K'}(x) - V_K(x)  = \sum_t A_K(x'_t, u'_t) \, .
\]
Also, for any $x$, the advantage is:
\begin{equation}\label{eq:advantage}
A_K(x,K'x) =   2x^\top(K'-K)^\top E_K x
+
x^\top(K'-K)^\top (R + B^\top P_KB) (K'-K)x \, .
\end{equation}
\end{lemma}

\begin{proof}
Let  $c_t'$ be the cost sequence generated by $K'$. Telescoping the
sum appropriately:
\begin{eqnarray*}
V_{K'}(x) - V_K(x)  &=&
\sum_{t=0} c_t' - V_K(x)  \\
&=&\sum_{t=0} (c_t'+ V_K(x'_t) - V_K(x'_t)) - V_K(x)  \\
&=&\sum_{t=0} (c_t'+ V_K(x'_{t+1}) - V_K(x'_t))  \\
&=&\sum_{t=0} A_K(x'_t, u'_t)
\end{eqnarray*}
which completes the first claim (the third equality uses the fact that $x = x_0 = x'_0$). 

For the second claim, observe that:
\begin{align*}
V_K(x)   =   x^\top\left(Q+K^\top R K\right) x  + 
 x^\top (A-BK)^\top P_K (A-BK) x
\end{align*}
And, for $u=K'x$,
\begin{eqnarray*}
A_K(x,u) & = &  Q_K(x,u) -V_K(x)\\
 & = &  x^\top\left(Q+(K')^\top R K'\right) x  + 
 x^\top (A-BK')^\top P_K (A-BK') x -V_K(x)\\
 & = &  x^\top\left(Q+(K'-K+K)^\top R (K'-K+K)\right) x  + \\
&& x^\top (A-BK- B(K'-K))^\top P_K (A -BK-B(K'-K)) x -V_K(x)\\
 & = &  2x^\top(K'-K)^\top \left( (R+ B^\top P_K B) K - B^\top P_K A\right) x + \\
&&x^\top (K'-K)^\top (R + B^\top P_KB) (K'-K)) x \, ,
\end{eqnarray*}
which completes the proof.
\end{proof}

This lemma is helpful in proving that $C(K)$ is gradient dominated.

\begin{lemma}\label{lemma:domination}
(Gradient domination, Lemma~\ref{lemma:domination:maintext} and Corollary~\ref{cor:gradientdom:main} restated) Let $K^*$ be an optimal policy. Suppose $K$ has
finite cost and $\mu >0$. It holds that:
%$C(\cdot)$ is gradient dominated in the following sense:
\begin{align*}
C(K)-C(K^*)  
& \leq 
\|\Sigma_{K^*}\| \Tr(E_K^\top (R + B^\top P_KB)^{-1}E_K)\\
& 
\leq
\frac{\|\Sigma_{K^*}\|}{\sigma_{\textrm{min}}(R)} \Tr(E_K^\top E_K)\\
& 
\leq
\frac{\|\Sigma_{K^*}\|}{\sigma_{min}(\Sigma_K)^2 \sigma_{\textrm{min}}(R)} \Tr(\nabla C(K)^\top \nabla C(K)) \\
& \leq\frac{\|\Sigma_{K^*}\|}{\mu^2 \sigma_{\textrm{min}}(R)} \Tr(\nabla C(K)^\top \nabla C(K))
\end{align*}
For a lower bound, it holds that:
\[
C(K)-C(K^*)   \geq \frac{\mu}{\|R + B^\top P_KB\|}  \, \Tr(E_K^\top E_K) 
\]
\end{lemma}

\begin{proof}
From Equation~\ref{eq:advantage} and by completing the square,
\begin{eqnarray}
&&Q_K(x,K'x) - V_K(x) \nonumber \\
& = & 2\Tr(xx^\top(K'-K)^\top E_K) +
                        \Tr(xx^\top(K'-K)^\top (R + B^\top P_KB) (K'-K)) \nonumber \\ 
&=& 
\Tr(xx^\top\left(K'-K+(R + B^\top P_KB)^{-1}E_K \right)^\top (R + B^\top P_KB)
    \left(K'-K+(R + B^\top P_KB)^{-1}E_K\right)) \nonumber \\ 
&& -\Tr(xx^\top E_K^\top(R + B^\top P_KB)^{-1}E_K) \nonumber \\ 
&\geq& -\Tr(xx^\top E_K^\top (R + B^\top P_KB)^{-1} E_K) \label{eq:completed_square}
\end{eqnarray}
with equality when $K'=K-(R + B^\top P_KB)^{-1}E_K$.

Let $x^*_t$ and $u^*_t$ be the sequence generated under $K_*$. Using this and Lemma~\ref{lemma:helper},
\begin{eqnarray*}
C(K) -C(K^*) & = & -\E \sum_t A_K(x^*_t,u^*_t)\\
&\leq& \E \sum_t \Tr(x^*_t(x^*_t)^\top E_K^\top (R + B^\top P_KB)^{-1} E_K)\\
&=& \Tr(\Sigma_{K^*}E_K^\top (R + B^\top P_KB)^{-1} E_K)\\
&\leq & \|\Sigma_{K^*}\| \Tr(E_K^\top (R + B^\top P_KB)^{-1}E_K)\\
&\leq &\|\Sigma_{K^*}\| \| (R + B^\top P_KB)^{-1}\| \,
\Tr(E_K^\top E_K)\\
&\leq&\frac{\|\Sigma_{K^*}\|}{\sigma_{\textrm{min}}(R)} \Tr(E_K^\top E_K)\\
&=&\frac{\|\Sigma_{K^*}\|}{ \sigma_{\textrm{min}}(R)}
    \Tr(\Sigma_K^{-1}\nabla C(K)^\top \nabla C(K) \Sigma_K^{-1})\\
&\leq&\frac{\|\Sigma_{K^*}\|}{\sigma_{min}(\Sigma_K)^2 \sigma_{\textrm{min}}(R)} \Tr(\nabla C(K)^\top \nabla C(K))\\
&\leq&\frac{\|\Sigma_{K^*}\|}{\mu^2 \sigma_{\textrm{min}}(R)} \Tr(\nabla C(K)^\top \nabla C(K))
\end{eqnarray*}
which completes the proof of the upper bound. Here the last step is because $\Sigma_K \succeq \E[x_0x_0^\top]$. 

For the lower bound, consider $K'=K-(R + B^\top P_KB)^{-1}E_K$ where 
equality holds in Equation~\ref{eq:completed_square}. Let $x'_t$ and
$u'_t$ be the sequence generated under $K'$. Using that $C(K^*)\leq
C(K')$, 
\begin{eqnarray*}
C(K) -C(K^*) &\geq& C(K) -C(K') \\
& = & -\E \sum_t A_K(x'_t,u'_t)\\
& = & \E \sum_t \Tr(x'_t(x'_t)^\top E_K^\top (R + B^\top P_KB)^{-1} E_K)\\
& \geq & \Tr( \Sigma_{K'} E_K^\top (R + B^\top P_KB)^{-1} E_K)\\
%& \geq & \mu \sigma_{\min}((R + B^\top P_KB)^{-1})  \Tr(E_K^\top E_K)\\
& \geq & \frac{\mu}{\|R + B^\top P_KB\|} \Tr(E_K^\top E_K)
\end{eqnarray*}
which completes the proof.
\end{proof}

Recall that a function $f$ is said to be smooth (or
$C^1$-smooth) if for some finite $\beta$, it satisfies:
\begin{equation}\label{eq:smoothness}
| f(x) - f(y) - \nabla f(y)^\top (x-y) | \leq \frac{\beta}{2}
\|x-y\|^2 \, .
\end{equation}
for all $x,y$ (equivalently, it is smooth if the gradients of $f$ are
continuous).

%Equation~\ref{eq:smoothness}.

\begin{lemma}\label{lemma:smoothness}
(``Almost'' smoothness, Lemma~\ref{lemma:smoothness:maintext} restated) $C(K)$ satisfies:
\[
C(K')-C(K) 
 = 
-2\Tr(\Sigma_{K'}(K-K')^\top E_K) +
\Tr(\Sigma_{K'}(K-K')^\top (R + B^\top P_KB) (K-K'))
\]
\end{lemma}

To see why this is related to smoothness (e.g. compare to
Equation~\ref{eq:smoothness}), suppose $K'$ is
sufficiently close to $K$ so that:
\begin{equation}\label{eq:taylor}
\Sigma_{K'} \approx \Sigma_K + O(\|K-K'\|)
\end{equation}
and the leading order term $2\Tr(\Sigma_{K'}(K'-K)^\top E_K)$ would
then behave as $\Tr((K'-K)^\top \nabla C(K))$.
The challenge in the proof (for gradient descent) is quantifying the
lower order terms in this argument.

\begin{proof}
The claim immediately results from Lemma~\ref{lemma:helper}, by using
Equation~\ref{eq:advantage} and taking an expectation.
\end{proof}

The next lemma spectral norm bounds on $P_K$ and $\Sigma_K$ are helpful:

\begin{lemma}\label{lemma:bounds}
It holds that:
\[
\|P_K\| \leq \frac{C(K)}{\mu} ,\quad \quad  \|\Sigma_K\| \leq \frac{C(K)}{\sigma_{\min}(Q)}
\]
\end{lemma}

\begin{proof}
For the first claim, $C(K)$ is lower bounded as:
\[
C(K) = \E_{x_0\sim \cD }  x_0^\top P_K x_0 \geq \|P_K\|\sigma_{\min}(\E x_0 x_0^\top)
\]
Alternatively, $C(K)$ can be lower bounded as:
\[
C(K) = \Tr(\Sigma_K (Q+K^\top RK)) \geq \Tr(\Sigma_K) \sigma_{\min}(Q) 
\geq \|\Sigma_K\|
\sigma_{\min}(Q) \, ,
\]
which proves the second claim.
\end{proof}

\subsection{Gauss-Newton Analysis}

The next lemma bounds the one step progress of Gauss-Newton.

\begin{lemma}(Lemma~\ref{lem:gaussnewton:main} restated)
Suppose that:
\[
K' = K - \eta (R+ B^\top P_K B)^{-1}  \nabla C(K) \Sigma_K^{-1}\, , .
\]
If $\eta \leq 1$, then
\[
C(K') -C(K^*) \leq \left(1- \frac{\eta \mu }{\|\Sigma_{K^*}\|}\right) (C(K)-C(K^*))
\]
\end{lemma}

\begin{proof}
First we prove this assuming $K'$ is a stabilizing policy (that is, $\rho(A-BK') < 1$). In this case we can apply Lemma~\ref{lemma:smoothness}.

Observe $K' = K - \eta (R+ B^\top P_K B)^{-1} E_K$.
Using Lemma~\ref{lemma:smoothness} and the condition on $\eta$,
\begin{eqnarray*}
C(K')-C(K) 
&=&
-2\eta\Tr(\Sigma_{K'} E_K^\top (R+ B^\top P_K B)^{-1} E_K) +
\eta^2  \Tr(\Sigma_{K'} E_K^\top (R+ B^\top P_K B)^{-1} E_K)\\
&\leq&-\eta \Tr(\Sigma_{K'} E_K^\top (R+ B^\top P_K B)^{-1}
    E_K)\\
&\leq&-\eta \sigma_{\textrm{min}}(\Sigma_{K'})\Tr(E_K^\top (R+ B^\top P_K B)^{-1}
    E_K)\\
& \leq & - \eta \mu\Tr(E_K^\top (R+ B^\top P_K B)^{-1} E_K) \\
& \leq & - \eta \frac{\mu}{\|\Sigma_{K^*}\|} (C(K)-C(K^*)) \, ,
\end{eqnarray*}
where the last step uses Lemma~\ref{lemma:domination}.

Now, we will prove that $K'$ is always stabilizing for our choice of $\eta$. We will use $K'(\eta)$ to denote the policy $K'$ when we choose step size $\eta$. Assume towards contradiction that for some $\eta \le 1$ $\Sigma_{K'}$ is not stabilizing. Let $\eta_0 = \inf_{\eta} \rho(A-BK'(\eta)) \ge 1$ and $\eta_1 = \eta_0 - \epsilon$ for small enough $\epsilon$. By definition $\eta_1$ is still stabilizing so we know $C(K'(\eta_1)) \le C(K)$, and also $\|A-BK'(\eta_1)\| \le \|A-BK\| + \|(R+ B^\top P_K B)^{-1}  \nabla C(K) \Sigma_K^{-1}\|$ is uniformly bounded for every $K'(\eta)$. By Lemma~\ref{lemma:SigmaK_perturbation} we know there exists a neighborhood of $K'(\eta_1)$ such that every policy in this neighborhood is stabilizing. However, this contradicts with the assumption that $K'(\eta_0)$ is not stabilizing when $\epsilon$ is chosen to be small enough.
\end{proof}

With this lemma, the proof of the convergence rate of the Gauss Newton
algorithm is immediate.

\begin{proof} (of Theorem~\ref{theorem:gd_exact}, Gauss-Newton case)
The theorem is due to that $\eta=1$ leads to a contraction of  $1- \frac{\eta \mu
}{\|\Sigma_{K^*}\|}$ at every step.
\end{proof}

\subsection{Natural Policy Gradient Descent Analysis}

The next lemma bounds the one step progress of the natural policy gradient.

\begin{lemma}\label{lemma:ngd}
Suppose:
\[
K' = K - \eta \nabla C(K) \Sigma_K^{-1}
\]
and that $\eta\leq\frac{1}{\|R + B^\top P_KB\|}$. It holds that:
\[
C(K') -C(K^*) \leq \left(1-\eta \sigma_{\textrm{min}}(R) \frac{\mu}{\|\Sigma_{K^*}\|}\right) (C(K)-C(K^*))
\]
\end{lemma}

\begin{proof}
We will again first prove the lemma when $K'$ is a stabilizing policy ($\rho(A-BK') < 1$). Using the same idea as in proof of Lemma~\ref{lem:gaussnewton:main} we can prove that $K'$ must be stabilizing for all the step sizes we choose.

Since $K' = K - \eta E_K$,
Lemma~\ref{lemma:smoothness} implies:
\[
C(K')-C(K) 
=
-2\eta\Tr(\Sigma_{K'} E_K^\top E_K) +
\eta^2  \Tr(\Sigma_{K'}E_K^\top (R + B^\top P_KB) E_K)
\]
The last term can be bounded as:
\begin{align*}
\Tr(\Sigma_{K'}E_K^\top (R + B^\top P_KB) E_K) &=
\Tr((R + B^\top P_KB) E_K\Sigma_{K'}E_K^\top) \\
&\leq \|R + B^\top P_KB\|\Tr(E_K\Sigma_{K'}E_K^\top) \\
&=  \|R + B^\top P_KB\|\Tr(\Sigma_{K'} E_K^\top E_K)\, .
\end{align*}
%using that for two positive semi-definite matrices $X,Y$, $\Tr(X Y)
%\leq \|X\| \Tr(Y)$.
Continuing and using the condition on $\eta$,
\begin{eqnarray*}
C(K')-C(K) 
%&=&
%-2\eta\Tr(\Sigma_{K'} E_K^\top E_K) +
%\eta^2  \Tr(\Sigma_{K'}E_K^\top (R + B^\top P_KB) E_K)\\
&\leq&
-2\eta\Tr(\Sigma_{K'} E_K^\top E_K) +
\eta^2  \|R + B^\top P_KB\|\Tr(\Sigma_{K'} E_K^\top E_K) \\
&\leq&
-\eta\Tr(\Sigma_{K'} E_K^\top E_K) \\
&\leq&
-\eta \sigma_{\textrm{min}}(\Sigma_{K'})\Tr(E_K^\top E_K) \\
& \leq & -\eta \mu\Tr(E_K^\top E_K) \\
& \leq & -\eta \frac{\mu \sigma_{\textrm{min}}(R)}{\|\Sigma_{K^*}\|} (C(K)-C(K^*))
\end{eqnarray*}
using Lemma~\ref{lemma:domination}.
\end{proof}

With this lemma, the proof of the natural policy
gradient convergence rate can be completed.

\begin{proof} (of Theorem~\ref{theorem:gd_exact}, natural policy gradient case)
Using Lemma~\ref{lemma:bounds},
\[
\frac{1}{\|R + B^\top P_KB\|} \geq \frac{1}{\|R\| + \|B\|^2\| P_K\|} \geq \frac{1}{\|R\| + \frac{\|B\|^2 C(K)}{\mu}}
\]
The proof is completed by induction: $C(K_1)\leq C(K_0)$, since
Lemma~\ref{lemma:ngd} can be applied.
The proof proceeds by arguing that Lemma~\ref{lemma:ngd} can be applied
at every step. If it were the case that $C(K_t)\leq C(K_0)$, then 
\[
\eta 
\leq
\frac{1}{\|R\| + \frac{\|B\|^2 C(K_0)}{\mu}}
\leq
\frac{1}{\|R\| + \frac{\|B\|^2 C(K_t)}{\mu}}
\leq
\frac{1}{\|R + B^\top P_{K_t}B\|}
\]
and by Lemma~\ref{lemma:ngd}:
\[
C(K_{t+1}) -C(K^*) \leq \left(1- \frac{\mu}{\|\Sigma_{K^*}\|} \, \frac{\sigma_{\textrm{min}}(R)}{\|R\| + \frac{\|B\|^2
      C(K_0)}{\mu}} \right) (C(K_t)-C(K^*))\,
\]
which completes the proof.
\end{proof}

\subsection{Gradient Descent Analysis}
\label{sec:gdanalysis}
As informally argued by Equation~\ref{eq:taylor}, the proof seeks to quantify
how $\Sigma_{K'}$ changes with $\eta$.  Then the proof bounds the one
step progress of gradient descent.

\subsubsection*{$\Sigma_{K}$ perturbation analysis}

This subsections aims to prove the following: 

\begin{lemma}\label{lemma:SigmaK_perturbation}
($\Sigma_{K}$ perturbation) Suppose $K'$ is such that:
%\begin{equation}\label{equation:condition}
\[
\|K'-K\|\leq \frac {\sigma_{\min}(Q) \mu}
{4 C(K) \|B\|\left( \|A-B K\|+ 1\right) }
%\frac{C(K) }{\sigma_{\min}(Q) \sigma_{\min}(\E x_0 x_0^\top)}
%\left( 2 \|A-B K\|\|B\|\|K-K'\|+ \|B\|^2\|K-K'\|^2\right) \leq 1/2
\]
%\end{equation}
It holds that:
\[
\|\Sigma_{K'} - \Sigma_K \| \leq 
4 \left(\frac{C(K)}{\sigma_{\min}(Q) } \right)^2
\frac{ \|B\| \left(\|A-B K\|+ 1\right)}{ \mu} \|K-K'\| 
\]
\end{lemma}

The proof proceeds by starting with a few technical lemmas. First, define a linear operator on
symmetric matrices, $\cT_K(\cdot)$, which can be viewed as a matrix on ${d+1 \choose 2}$ dimensions.
Define this operator on a symmetric matrix $X$ as follows:
\[
\cT_K( X) := \sum_{t=0}^\infty (A-BK)^t X [(A-BK)^{\top}]^t
\]
Also define the induced norm of $\cT$ as follows:
\begin{equation} \label{equation:induced}
\|\cT_K\| = \sup_X \frac{\|\cT_K( X) \|}{\|X\|} 
\end{equation}
where the supremum is over all symmetric matrices $X$ (whose spectral
norm is non-zero).

Also, define 
\[
\Sigma_0 = \E x_0 x_0^\top
\].  

\begin{lemma}\label{lemma:operator}
($\cT_K$ norm bound) It holds that
\[
\|\cT_K\| \leq \frac{C(K)}{\mu \,  \sigma_{\min}(Q)}
\]
\end{lemma}

\begin{proof}
For a unit norm vector $v\in \bR^d$ and unit spectral norm matrix $X$, 
\begin{eqnarray*}
v^\top (\cT_K( X) ) v & = & \sum_{t=0}^\infty v^\top (A-BK)^t X
                        [(A-BK)^{\top}]^t v \\
& = & 
\sum_{t=0}^\infty \Tr([(A-BK)^{\top}]^t v v^\top (A-BK)^t X)\\
& = & 
\sum_{t=0}^\infty \Tr([\Sigma_0^{1/2} (A-BK)^{\top}]^t v
      v^\top(A-BK)^t \Sigma_0^{1/2} \Sigma_0^{-1/2} X \Sigma_0^{-1/2})
  \\
& \leq & 
\sum_{t=0}^\infty \Tr([\Sigma_0^{1/2} (A-BK)^{\top}]^t v
      v^\top(A-BK)^t \Sigma_0^{1/2}) \|\Sigma_0^{-1/2} X
         \Sigma_0^{-1/2}\|\\
& = & 
\|\Sigma_0^{-1/2} X \Sigma_0^{-1/2}\| \, \left(v^\top \cT_K(\Sigma_0) v\right)\\
& \leq & 
\frac{1}{\sigma_{\min}(\E x_0 x_0^\top)} \| \cT_K(\Sigma_0) \|\\
& = & 
\frac{1}{\mu} \| \Sigma_K \|
\end{eqnarray*}
using that $\cT_K(\Sigma_0) = \Sigma_K$. The proof
is completed using the upper bound on $\| \Sigma_K \|$ in Lemma~\ref{lemma:bounds}.
\end{proof}

Also, with respect to $K$, define another linear operator on symmetric
matrices: 
\[
\cF_K(X) = (A-BK)X(A-BK)^\top \, .
\]
Let $\Id$ to denote the identity operator on the
same space. 
Define the induced norm $\|\cdot \|$
of these operators as in Equation~\ref{equation:induced}. Note these operators are related to the operator $\cT_K$
as follows: 

\begin{lemma}
When $(A-BK)$ has spectral radius smaller than 1,
$$\cT_K = (\Id - \cF_K)^{-1}.$$
\end{lemma}

\begin{proof}
When $(A-BK)$ has spectral radius smaller than 1, $\cT_K$ is well defined and is the solution of $\cT_K = \Id + \cT_K\circ\cF_K$. Therefore $\cT_K\circ(\Id - \cF_K) = \Id$ and $\cT_K = (\Id - \cF_K)^{-1}$.
\end{proof}

Since,
\[
\Sigma_K =  \cT_K  (\Sigma_0) =  (\Id - \cF_K)^{-1} (\Sigma_0) \, .
\]
The proof of Lemma~\ref{lemma:SigmaK_perturbation} seeks to bound:
\[
\|\Sigma_K - \Sigma_{K'}\| =  \|(\cT_K - \cT_{K'})( \Sigma_0) \|
= \| ((\Id - \cF_K)^{-1} -(\Id - \cF_{K'})^{-1} ) (\Sigma_0) \|\, .
\]
The following two perturbation bounds are helpful in this.

\iffalse
\begin{proof}
$$\Sigma_K =\E[ \sum_{i=0}^\infty x_ix_i^\top] = \sum_{i=0}^\infty (A-BK)^{i}\E[x_0x_0^\top][(A-BK)^{\top}]^i = \sum_{i=0}^\infty \cF_K^i(\Sigma_0) = (\Id- \cF_K)^{-1}\Sigma_0.
$$
\end{proof}
\fi

\begin{lemma}\label{lemma:FK_perturbation}
It holds that:
$$
\|\cF_K - \cF_{K'}\| \le 2\|A-BK\| \|B\|\|K - K'\|+\|B\|^2\|K - K'\|^2.
$$
\end{lemma}

\begin{proof}
Let $\Delta=K - K'$. For every matrix $X$, 
$$
(\cF_K - \cF_{K'})(X) = (A-BK)X(B\Delta)^\top + (B\Delta)X(A-BK)^\top - (B\Delta)X(B\Delta)^\top.
$$
The operator norm of $\cF_K - \cF_{K'}$ is the maximum possible ratio
in spectral norm of $(\cF_K - \cF_{K'})(X)$ and $X$. Then the claim
follows because $\|AX\| \le \|A\|\|X\|$.  
\end{proof}

\begin{lemma}\label{lemma:operator_perturbation}
If 
\[
\|\cT_K\|\|\cF_K -\cF_{K'}\| \le 1/2 \, ,
\]
and both $\cF_K$ and $\cF_{K'}$ satisfy $\rho(\cF_K) < 1$ and $\rho(\cF_{K'}) < 1$ then 
\begin{eqnarray*}
\|\left(\cT_K - \cT_{K'}\right) (\Sigma)\| &\le &
2\|\cT_K\| \|\cF_K -\cF_{K'}\| \|\cT_K(\Sigma)\|.\\
&\le &
2\|\cT_K\|^2 \|\cF_K -\cF_{K'}\| \|\Sigma\|.\\
\end{eqnarray*}
%If $\|\cA^{-1}\|\|\cB\| \le 1/2$, then we have
%$$
%\|[\cA^{-1} - (\cA-\cB)^{-1}](\Sigma)\| \le 2\|\cA^{-1}\|\|\cB\|\|\cA^{-1}(\Sigma)\|.
%$$
\end{lemma}

\begin{proof}
Define $\cA = \Id - \cF_K$, and $\cB = \cF_{K'}-\cF_K$. In this case
$\cA^{-1} = \cT_K$ and $(\cA-\cB)^{-1} = \cT_{K'}$. Hence, the
  condition $\|\cT_K\|\|\cF_K-\cF_{K'}\| \le 1/2$ translates to the
  condition $\|\cA^{-1}\|\|\cB\| \le 1/2$.

Observe:
\[
(\cA^{-1} - (\cA-\cB)^{-1})(\Sigma) 
= (\Id - (\Id - \cA^{-1}\circ\cB)^{-1})(\cA^{-1}(\Sigma))
=(\Id - (\Id - \cA^{-1}\circ\cB)^{-1})(\cT_K(\Sigma)) \, .
\]
Since $(\Id - \cA^{-1}\circ\cB)^{-1} = \Id + \cA^{-1}\circ\cB\circ(\Id
- \cA^{-1}\circ\cB)^{-1}$, 
\[
\|(\Id - \cA^{-1}\circ\cB)^{-1}\| 
\leq 1 + \|\cA^{-1}\circ\cB\| \|(\Id - \cA^{-1}\circ\cB)^{-1}\| 
\leq 1 + 1/2 \|(\Id - \cA^{-1}\circ\cB)^{-1}\| 
\]
which implies $\|(\Id - \cA^{-1}\circ\cB)^{-1}\|\leq 2$. Hence,
\[
\|\Id - (\Id - \cA^{-1}\circ\cB)^{-1}\| = \|\cA^{-1}\circ\cB\circ (\Id - \cA^{-1}\circ\cB)^{-1}\| \le \|\cA^{-1}\|\|\cB\|\|(\Id - \cA^{-1}\circ\cB)^{-1}\| = 2\|\cA^{-1}\|\|\cB\|.
\]
and  so
\[
\|\Id - (\Id - \cA^{-1}\circ\cB)^{-1}\| \le 2\|\cA^{-1}\|\|\cB\| =
2\|\cT_K\|\|\cF_K-\cF_{K'}\| \, .
\] 
Combining these two,
\[
\|\left(\cT_K - \cT_{K'}\right) (\Sigma)\| \le \|(\Id - (\Id - \cA^{-1}\circ\cB)^{-1})\| \|\cT_K(\Sigma)\| \le 
2\|\cT_K\|\|\cF_K-\cF_{K'}\|\|\cT_K(\Sigma)\|.
\]

This proves the main inequality. The last step of the inequality is
just applying definition of the norm of $\cT_K$: $\|\cT_K(\Sigma)\|
\le \|\cT_K\|\|\Sigma\|$. 
\end{proof}

With these Lemmas, we can first prove a weaker version of Lemma~\ref{lemma:SigmaK_perturbation} which assumes $\cF_{K'}$ has spectral radius at most 1,

\begin{lemma}\label{lemma:weaker_SigmaK_perturbation}
Lemma~\ref{lemma:SigmaK_perturbation} holds with the additional assumption that $\rho(\cF_{K'}) < 1$ (where $\cF_{K'}$ is defined as $\cF_{K'}(X) = (A-BK')X(A-BK')^\top$).
\end{lemma}

\begin{proof} %(of Lemma~\ref{lemma:SigmaK_perturbation})
%\sk{We need to check this.  The below is from the previous lemma.}
First, the proof shows $\|\cT_K\|\|\cF_K -\cF_{K'}\| \le 1/2 $,
which is the desired condition in
Lemma~\ref{lemma:operator_perturbation}. 
%The assumed condition, Equation~\ref{equation:condition}, can be re-written as:
%\[
%\frac { C(K) \left( 2\|A-B K\|\|B\|\|K'-K\|+ 2 \|B\|\|K'-K\|\right) }{\sigma_{\min}(Q) \mu} \leq \frac{1}{2}
%\]
First, observe that under the assumed condition on $\|K-K'\|$, implies that
\[
\|B\| \|K'-K\|\leq \frac {\sigma_{\min}(Q) \mu}
{4 C(K) \left( \|A-B K\|+ 1\right) }\leq \frac{1}{4} \frac{\sigma_{\min}(Q) \mu}{C(K)}\leq \frac{1}{4}
\]
using that $\frac{\sigma_{\min}(Q) \mu}{C(K)}\leq 1$ due
to Lemma~\ref{lemma:bounds}. Using Lemma~\ref{lemma:FK_perturbation},
\begin{align}
\|\cF_K -\cF_{K'}\| 
&\leq \left( 2 \|A-B K\|\|B\|\|K-K'\|+\|B\|^2\|K-K'\|^2\right)
  \nonumber \\
& \leq 2\|B\| \left( \|A-B K\|+ 1\right) \|K-K'\| \label{equation:FK_bound}
\end{align}
Using this and Lemma~\ref{lemma:operator}, 
\begin{eqnarray*}
\|\cT_K\|\|\cF_K -\cF_{K'}\| 
 \leq \frac{C(K)}{\sigma_{\min}(Q) \mu}
2\|B\| \left( \|A-B K\|+ 1\right) \|K-K'\|
\leq  \frac{1}{2}
\end{eqnarray*}
where the last step uses the condition on $\|K-K'\|$.

Thus,
\begin{align*}
\|\Sigma_{K'} - \Sigma_K \| 
&\leq 2\|\cT_K\| \|\cF_K -\cF_{K'}\| \|\cT_K(\Sigma_0)\| \\
& \leq
2 \frac{C(K)}{\sigma_{\min}(Q) \mu} \,
\left(2\|B\| \left( \|A-B K\|+ 1\right) \|K-K'\|\right) \, \frac{C(K)}{\sigma_{\min}(Q) } \\
\end{align*}
using Lemmas~\ref{lemma:bounds} and~\ref{lemma:FK_perturbation}.
\end{proof}

Now the only remaining step to prove Lemma~\ref{lemma:SigmaK_perturbation} is to show that within the ball assumed in Lemma~\ref{lemma:SigmaK_perturbation}, the policy is always stabilizing (that is, $\rho(\cF_{K'}) < 1$).

\begin{lemma}
Suppose $K'$ is such that:
%\begin{equation}\label{equation:condition}
\[
\|K'-K\|\leq \frac {\sigma_{\min}(Q) \mu}
{4 C(K) \|B\|\left( \|A-B K\|+ 1\right) }
%\frac{C(K) }{\sigma_{\min}(Q) \sigma_{\min}(\E x_0 x_0^\top)}
%\left( 2 \|A-B K\|\|B\|\|K-K'\|+ \|B\|^2\|K-K'\|^2\right) \leq 1/2
\]
%\end{equation}
It holds that $\rho(\cF_{K'}) < 1$ (where $\cF_{K'}$ is defined as $\cF_{K'}(X) = (A-BK')X(A-BK')^\top$).\label{lemma:stabilizingball}
\end{lemma}

Before proving this lemma we first claim that whenever $\rho(\cF_{K'})$ is very close to $1$, the final covariance matrix $\Sigma_{K'}$ must be large. Note that $\rho(\cF_{K'}) = \rho(A-BK')^2$ so we only need to prove this for $A-BK'$.

\begin{lemma}\label{lemma:varianceupper}
For any $K'$ with $\rho(A-BK') < 1$, we have
\[
\mathrm{tr}(\Sigma_{K'}) \ge \frac{\mu}{2(1-\rho(A-BK'))}.
\]
\end{lemma}

\begin{proof}
We know
\[
\Sigma_{K'} = \sum_{i=0}^\infty \cF_{K'}^{i}(\Sigma_0).
\]
Since $\Sigma_0\preceq \mu I$, we know the $i$-th term $\cF_{K'}^{i}(\Sigma_0) \succeq \mu(A-BK')^i[(A-BK')^\top)]^i$. The trace of this term is at least 
\[
\mathrm{tr}(\cF_{K'}^{i}(\Sigma_0)) \ge \mu \mathrm{tr}((A-BK')^i[(A-BK')^\top)]^i) \ge \mu \|(A-BK')^i\|_F^2 \ge \mu \rho((A-BK')^i)^2 = \mu\rho(A-BK')^{2i}.
\]
Now summing over the trace of all the terms gives us the result. 
\end{proof}

Now we are ready to prove Lemma~\ref{lemma:stabilizingball}.
\begin{proof}
Let $\Gamma = \mbox{tr}(\Sigma_K) + d \left(\frac{C(K)}{\sigma_{\min}(Q) } \right)$, this is the maximum possible value for $\mathrm{tr}(\Sigma_{K'})$ according to Lemma~\ref{lemma:weaker_SigmaK_perturbation} when $K'$ is close to $K$ as in Lemma~\ref{lemma:SigmaK_perturbation} and $\rho(\cF_{K'}) < 1$.  Let $\epsilon = 3/\mu\Gamma$. We know that $\rho(\Sigma_K) < 1-\epsilon$ because otherwise it contradicts with Lemma~\ref{lemma:varianceupper}. 

Assume towards contradiction that there exists a $K'$ within the ball $\|K'-K\| \le \frac {\sigma_{\min}(Q) \mu}
{4 C(K) \|B\|\left( \|A-B K\|+ 1\right) }$ such that $\rho(A-BK') \ge 1$, since spectral radius is a continuous function~\citep{tyrtyshnikov2012brief}, we know there must be a point $K''$ on the path between $K$ and $K'$ such that $\rho(K'') = 1 - \epsilon$. Now for $K''$, we can apply Lemma~\ref{lemma:weaker_SigmaK_perturbation}, and conclude that $\|\Sigma_{K''} - \Sigma_K\| \le 4 \left(\frac{C(K)}{\sigma_{\min}(Q) } \right)^2
\frac{ \|B\| \left(\|A-B K\|+ 1\right)}{ \mu} \|K-K''\|\le \left(\frac{C(K)}{\sigma_{\min}(Q) } \right)$. As a result, $\mathrm{tr}(K'') \le \mathrm{tr}(K) + d\|\Sigma_{K''} - \Sigma_K\| \le \Gamma$. On the other hand, by Lemma~\ref{lemma:varianceupper} we know $\mathrm{tr}(K'') > 1.5 \Gamma$. This is a contradiction. Therefore for any point $K'$ within the ball we have $\sigma(A-BK') < 1$. 
\end{proof}

Lemma~\ref{lemma:SigmaK_perturbation} now follows immediately from Lemma~\ref{lemma:weaker_SigmaK_perturbation} and Lemma~\ref{lemma:stabilizingball}.

\subsubsection*{Gradient Descent Progress}

Equipped with these lemmas, the one step progress of gradient descent
can be bounded.

\begin{lemma}\label{lemma:gd}
Suppose that 
\[
K' = K - \eta \nabla C(K)\, ,
\] 
where
\begin{equation}\label{eq:gd_condition}
\eta \leq \frac{1}{16} \min\left\{
\left(\frac{\sigma_{\min}(Q) \mu}{C(K)} \right)^2
\frac{1}{ \|B\| \| \nabla C(K)\| (1+\|A-B K\|) }
\, ,
\,
\frac{\sigma_{\min}(Q)}{2 C(K) \|R + B^\top P_KB\|}
\right\} \, .
\end{equation}
It holds that:
\[
C(K') -C(K^*) \leq \left(1-\eta \sigma_{\textrm{min}}(R) \frac{\mu^2}{\|\Sigma_{K^*}\|}\right) (C(K)-C(K^*))
\]
\end{lemma}

\begin{proof}
By Lemma~\ref{lemma:smoothness},
\begin{eqnarray*}
&&C(K')-C(K)\\
  &= &
-2\eta\Tr(\Sigma_{K'} \Sigma_K E_K^\top E_K) +
 \eta^2\Tr(\Sigma_K \Sigma_{K'}\Sigma_KE_K^\top (R + B^\top P_KB)
                 E_K)\\
& \leq & 
-2\eta\Tr(\Sigma_K E_K^\top E_K \Sigma_K) +
2\eta\|\Sigma_{K'}-\Sigma_K\| \Tr(\Sigma_K E_K^\top E_K) \\
&&+
\eta^2\|\Sigma_{K'}\| \|R + B^\top P_KB\| \Tr(\Sigma_K \Sigma_KE_K^\top  E_K)\\
& \leq & 
-2\eta\Tr(\Sigma_K E_K^\top E_K \Sigma_K) +
2\eta\frac{\|\Sigma_{K'}-\Sigma_K\|}{\sigma_{\min}(\Sigma_K)} \Tr(\Sigma_K E_K^\top E_K \Sigma_K)  \\
&&+
\eta^2\|\Sigma_{K'}\| \|R + B^\top P_KB\| \Tr(\Sigma_K E_K^\top E_K \Sigma_K) \\
& = & 
-2\eta \left(
1 - \frac{\|\Sigma_{K'}-\Sigma_K\|}{\sigma_{\min}(\Sigma_K)} -\frac{\eta}{2}\|\Sigma_{K'}\| \|R + B^\top P_KB\|
\right)
\Tr(\nabla C(K)^\top \nabla C(K)) \\
%& \leq & 
%-2\eta \mu^2 \left(
%1 - \frac{\|\Sigma_{K'}-\Sigma_K\|}{\sigma_{\min}(\Sigma_K)} -\frac{\eta}{2}\|\Sigma_{K'}\| \|R + B^\top P_KB\|
%\right)
%\Tr(E_K^\top E_K)\\
& \leq & 
-2\eta \frac{\mu^2 \sigma_{\textrm{min}}(R)}{\|\Sigma_{K^*}\|} \left(
1 - \frac{\|\Sigma_{K'}-\Sigma_K\|}{\mu} -\frac{\eta}{2}\|\Sigma_{K'}\| \|R + B^\top P_KB\|
\right)
(C(K)-C(K^*)) 
\end{eqnarray*}
where the last step uses Lemma~\ref{lemma:domination}.

By Lemma~\ref{lemma:SigmaK_perturbation},
\begin{eqnarray*}
\frac{\|\Sigma_{K'} - \Sigma_K \|}{\mu} \leq 
4\eta \left(\frac{C(K)}{\sigma_{\min}(Q) \mu} \right)^2
\|B\| \left(\|A-B K\|+ 1) \right) \|\nabla C(K)\|
\leq 
1/4
\end{eqnarray*}
using the assumed condition on $\eta$.

Using this last claim and Lemma~\ref{lemma:bounds},
\[
\|\Sigma_{K'} \| \leq \|\Sigma_{K'} - \Sigma_K \| +\| \Sigma_K\|
\leq \frac{ \mu}{4}+\frac{C(K)}{\sigma_{\min}(Q)}
\leq \frac{\|\Sigma_{K'} \|}{4}+\frac{C(K)}{\sigma_{\min}(Q)}
\]
and so $\|\Sigma_{K'} \| \leq \frac{4C(K)}{3\sigma_{\min}(Q)}$. Hence,
\[
1 - \frac{\|\Sigma_{K'}-\Sigma_K\|}{\mu} -\frac{\eta}{2}\|\Sigma_{K'}\| \|R + B^\top P_KB\|
\geq 1-1/4-\frac{\eta}{2} \frac{4C(K)}{3\sigma_{\min}(Q)} \|R + B^\top P_KB\| \geq 1/2
\]
using the condition on $\eta$.
\end{proof}

In order to prove a gradient descent convergence
rate, the following bounds are helpful:
\begin{lemma}\label{lemma:more_bounds}
It holds that
\[
\|\nabla C(K)\| \leq  \frac{C(K)}{\sigma_{\min}(Q)} \sqrt{\frac{ \|R + B^\top P_KB\| (C(K)-C(K^*))}{\mu} }
\]
and that:
\[
\|K\| \leq \frac{1}{\sigma_{\min}(R)}
\left(\sqrt{\frac{ \|R + B^\top P_KB\| (C(K)-C(K^*))}{\mu} }
+\|B^\top P_KA\|\right)
\]
\end{lemma}

\begin{proof}
Using Lemma~\ref{lemma:bounds},
\[
\|\nabla C(K)\|^2 \leq \Tr(\Sigma_K E_K^\top E_K \Sigma_K) \leq
\|\Sigma_K\|^2 \Tr(E_K^\top E_K) \leq \left(
  \frac{C(K)}{\sigma_{\min}(Q)} \right)^2 \Tr(E_K^\top E_K)
\]
By Lemma~\ref{lemma:domination},
\[
\Tr(E_K^\top E_K) \leq \frac{ \|R + B^\top P_KB\| (C(K)-C(K^*))}{\mu} 
\]
which proves the first claim.

Again using Lemma~\ref{lemma:domination},
\begin{eqnarray*}
\|K\| & \leq & \|(R + B^\top P_KB)^{-1}\|\|(R + B^\top P_KB)K\|\\
& \leq & \frac{1}{\sigma_{\min}(R)} \|(R + B^\top P_KB)K\|\\
& \leq & \frac{1}{\sigma_{\min}(R)} \left(\|(R + B^\top P_KB)K-B^\top P_KA\|+\|B^\top P_KA\|\right) \\
& = & \frac{\|E_K\|}{\sigma_{\min}(R)}
         +\frac{\|B^\top P_KA\|}{\sigma_{\min}(R)} \\
& \leq & \frac{\sqrt{\Tr(E_K^\top E_K)}}{\sigma_{\min}(R)}
         +\frac{\|B^\top P_KA\|}{\sigma_{\min}(R)} \\
& = & \frac{\sqrt{(C(K)-C(K^*)) \|R + B^\top P_KB\| }}{\sqrt{\mu}\sigma_{\min}(R)}  
         +\frac{\|B^\top P_KA\|}{\sigma_{\min}(R)} 
\end{eqnarray*}
which proves the second claim.
\end{proof}

With these lemmas, the proof of the  gradient descent convergence rate
follows:

\begin{proof} (of Theorem~\ref{theorem:gd_exact}, gradient descent case)
First, the following argues that progress is made at $t=1$.
Based on Lemma~\ref{lemma:bounds} and Lemma~\ref{lemma:more_bounds}, by choosing $\eta$ to be an
appropriate polynomial in $\frac{1}{C(K_0)},\frac{1}{\|A\|},\frac{1}{\|B\|},\frac{1}{\|R\|},$ $
\sigma_{\textrm{min}}(R),\sigma_{\textrm{min}}(Q)$ and $\mu$, 
the stepsize condition in
Equation~\ref{eq:gd_condition} is satisfied. Hence, by
Lemma~\ref{lemma:gd},
\[
C(K_1) -C(K^*) \leq \left(1-\eta \sigma_{\textrm{min}}(R) \frac{\mu^2}{\|\Sigma_{K^*}\|}\right) (C(K_0)-C(K^*))
\]
which implies that the cost decreases at $t=1$. Proceeding
inductively, now suppose that $C(K_t)\leq C(K_0)$, then the stepsize condition in
Equation~\ref{eq:gd_condition} is still satisfied (due to the use
of $C(K_0)$ in bounding the quantities in
Lemma~\ref{lemma:more_bounds}). Thus, 
Lemma~\ref{lemma:gd} can again be applied for the update at time $t+1$ to obtain:
\[
C(K_{t+1}) -C(K^*) \leq \left(1-\eta \sigma_{\textrm{min}}(R)
  \frac{\mu^2}{\|\Sigma_{K^*}\|}\right) (C(K_t)-C(K^*)) \, .
\]
Provided 
\[
T \geq \frac{\|\Sigma_{K^*}\|}{\eta \mu^2 \sigma_{\textrm{min}}(R)}
\log \frac{C(K_0) -C(K^*)}{\eps} \, ,
\]
then $C(K_T) -C(K^*) \leq \eps$, and the result follows.
\end{proof}

\section{Analysis: the Model-free case}
\label{sec:modelfree}

This section shows how techniques from zeroth order optimization allow
the algorithm to run in the model-free setting with only black-box
access to a simulator. The dependencies on various parameters are not
optimized, and the notation $\ph$ is used to represent different polynomial
factors in the relevant factors ($\frac{C(K_0)}{\mu\sigma_{min}(Q)},
\|A\|, \|B\|, \|R\|, 1/\sigma_{min}(R)$). When the polynomial also
depend on dimension $d$ or accuracy $1/\epsilon$, this is specified 
as parameters ($\ph(d,1/\epsilon)$).  

The section starts by showing how the infinite horizon can be approximated
with a finite horizon.

\subsection{Approximating $C(K)$ and $\Sigma_K$ with finite horizon}
\label{sec:finitehorizonsimulation}

This section shows that as long as there is an upper bound on
$C(K)$, it is possible to approximate both $C(K)$ and $\Sigma(K)$ with
any desired accuracy. 

%SK: added an expectation for C(K) that missing in the lemma statement
\begin{lemma}\label{lemma:finitehorizon}
For any $K$ with finite $C(K)$, let
$\Sigma^{(\it)}_K=\E[\sum_{i=0}^{\it-1} x_ix_i^\top]$ and
$C^{(\it)}(K) = \E[\sum_{i=0}^{\it-1} x_i^\top Qx_i+u_i^\top R u_i] =
\langle \Sigma^{(t)}_K, Q+K^\top RK\rangle$. 
If 
$$
\it \ge \frac{d\cdot C^2(K)}{\epsilon\mu \sigma_{min}^2(Q)},
$$
then $\|\Sigma^{(\it)}_K - \Sigma_K\| \le \epsilon$.  
Also, if 
$$
\it \ge \frac{d\cdot C^2(K)(\|Q\|+\|R\|\|K\|^2)}{\epsilon\mu \sigma_{min}^2(Q)}
$$
then $C(K) \ge C^{(\it)}(K) \ge C(K) - \epsilon$.
\end{lemma}

\begin{proof}
  First, the bound on $\Sigma_K$ is proved. Define the operators
  $\cT_K$ and $\cF_K$ as in Section~\ref{sec:gdanalysis}, observe
  $\Sigma_K = \cT_K(\Sigma_0)$ and
  $\Sigma^{(\it)}_K = \Sigma_K - (\cF_K)^\it(\Sigma_K)$.

If $X\succeq Y$, then $\cF_K(X) \succeq \cF_K(Y)$, this follows
immediately from the form of $\cF_K(X) = (A+BK)X(A+BK)^\top$. If $X$
is PSD then $WXW^\top$ is also PSD for any $W$. 

Now, since
$$
\sum_{i=0}^{\it-1} \mathrm{tr}(\cF^\it(\Sigma_0)) = \mathrm{tr}(\sum_{i=0}^{\it-1} \cF^\it(\Sigma_0)) \le \mathrm{tr}(\sum_{i=0}^\infty \cF^\it(\Sigma_0)) = \mathrm{tr}(\Sigma_K) \le \frac{d\cdot C(K)}{\sigma_{min}(Q)}.
$$
(Here the last step is by Lemma~\ref{lemma:bounds}), and all traces are nonnegative, then there must exists $j < \it$ such that $\mathrm{tr}(\cF_K^j(\Sigma_0)) \le \frac{d\cdot C(K)}{\it\sigma_{min}(Q)}.$

Also, since $\Sigma_K \preceq \frac{C(K)}{\mu \sigma_{min}(Q)} \Sigma_0$, 
$$
\mathrm{tr}(\cF_K^j(\Sigma_K)) \le \frac{C(K)}{\mu \sigma_{min}(Q)}\mathrm{tr}(\cF_K^j(\Sigma_0)) \le \frac{d\cdot C^2(K)}{\it\mu \sigma_{min}^2(Q)}.
$$

Therefore as long as 
\[
\it \ge  \frac{dC^2(K)}{\epsilon\mu
  \sigma_{min}^2(Q)},
\]
it follows that:
$$
\|\Sigma_K - \Sigma^{(\it)}_K\| \le \|\Sigma_K - \Sigma^{(j)}_K\| = \|\cF_K^j(\Sigma_K)\| \le \epsilon.
$$
Here the first step is again because of all the terms are PSD, so
using more terms is always better. The last step follows because
$\cF_K^j(\Sigma_K)$ is also a PSD matrix so the spectral norm is bounded
by trace. In fact,  it holds that $\mathrm{tr}(\Sigma_K-\Sigma^{(\it)}_K)$ is smaller than $\epsilon$.

Next, observe $C^{(\it)}(K) = \langle \Sigma^{(\it)}_K, Q+K^\top RK\rangle$ and $C(K) = \langle \Sigma_K, Q+K^\top RK\rangle$, therefore
$$
C(K) - C^{(\it)}(K) \le \mathrm{tr}(\Sigma_K-\Sigma^{(\it)}_K) (\|Q\|+\|R\|\|K\|^2).
$$
Therefore if 
$$
\it \ge \frac{d\cdot C^2(K)(\|Q\|+\|R\|\|K\|^2)}{\epsilon\mu \sigma_{min}^2(Q)},
$$
then $\mathrm{tr}(\Sigma_K-\Sigma^{(\it)}_K) \le \epsilon/(\|Q\|+\|R\|\|K\|^2)$ and hence $C(K) - C^{(\it)}(K) \le \epsilon$.

\end{proof}

%\subsection{Analysis: the Model-free case}

%Another simplifying assumption that we will make here is that the simulator is run with infinitely many iterations. However, this assumption can be removed as in Section~\ref{sec:finitehorizonsimulation} we show that it suffices to simulate a polynomial number of iterations in order to get the desired accuracy.

\subsection{Perturbation of $C(K)$ and $\nabla C(K)$}

%\sk{I ``shortened'' the final bound a little using an upper
%  bound. I'll check it later and add in any missing steps, if need be.}

The next lemma show that the function value and its gradient are
approximate preserved if a small perturbation to the policy $K$ is applied.  
  
\begin{lemma}\label{lemma:objectiveperturbation}
($C_{K}$ perturbation) 
Suppose $K'$ is such that:
\[
\|K'-K\|\leq \min \left(\frac {\sigma_{\min}(Q) \mu}
{4 C(K) \|B\|\left( \|A-B K\|+ 1\right) }, \|K\|\right)
\]
then:
\begin{eqnarray*}
&& |C(K')-C(K)| \\
&\leq & 6 \|K\| \, \|R\| \, \E \|x_0\|^2 \,
\left(\frac{C(K)}{\mu \,  \sigma_{\min}(Q)}\right)^2 
\left( \|K\|\|B\| \|A-B K\|+ \|K\|\|B\|
+ 1 \right) \|K-K'\|
\end{eqnarray*}
\end{lemma}

\iffalse
\begin{eqnarray*}
&& |C(K')-C(K)| \\
&\leq & 6 \E \|x_0\|^2\left( 
 \left(\frac{C(K)}{\mu \,  \sigma_{\min}(Q)}\right)^2 \|K\|^2\|R\|\|B\| \left( \|A-B K\|+ 1\right)\
+ \left(\frac{C(K)}{\mu \,  \sigma_{\min}(Q)}\right) \|K\|\|R\|
\right) \|K-K'\|
%4 \left(\frac{C(K)}{\sigma_{\min}(Q) } \right)^2
%\frac{ \|B\| \left(\|A-B K\|+ 1\right)}{ \mu} \|K-K'\| 
\end{eqnarray*}
\fi

\begin{proof}
As in the proof of Lemma~\ref{lemma:FK_perturbation}, the assumption
implies that $\|\cT_K\| \|\cF_K -\cF_{K'}\| \leq 1/2$, and, from Equation~\ref{equation:FK_bound}, that
\[
\|\cF_K -\cF_{K'}\| \leq 2\|B\| \left( \|A-B K\|+ 1\right) \|K-K'\|
\]
First, observe:
\begin{eqnarray*}
C(K') - C(K) &\leq &\Tr(\E x_0 x_0^\top) 
\|\cT_{K'}(Q+(K')^\top R K') - \cT_K(Q+K^\top R K)\|\\
&= & \E \|x_0\|^2 \|\cT_{K'}(Q+(K')^\top R K') - \cT_K(Q+K^\top R K)\|\\
& = &\E \|x_0\|^2 \|P_{K'} - P_k\|.
\end{eqnarray*}
%\rg{I think the above is not correct. The objective is the
%  inner-product of $\Sigma_0$ and $P_K$ (the long term is just
%  $P_{K'}-P_K$). The inner-product $\langle A, B\rangle \le
%  \mathrm{tr}(A)\|B\|$ when $A$ is PSD, but I don' think  $\langle A,
%  B\rangle \le \|A\|\|B\|$ is true. We should just replace $\|\E x_0
%  x_0^\top\|$ with trace.} \sk{Oops. It was late. Now it should be fine.}

To bound the difference we just need to bound $\|P_{K'} - P_k\|$. For that we have
\begin{eqnarray*}
& & P_{K'} - P_K \\
&= & \|\cT_{K'}(Q+(K')^\top R K') - \cT_K(Q+K^\top R K)\|\\
& \leq & \|\cT_{K'}(Q+(K')^\top R K')-\cT_K(Q+(K')^\top R K') \\
&&     -    \left(\cT_K(Q+K^\top R K) - \cT_K(Q+(K')^\top R
   K')\right)\|\\
& = & \|\cT_{K'}(Q+(K')^\top R K')-\cT_K(Q+(K')^\top R K') 
   -    \cT_K \circ (K^\top R K-(K')^\top R K')\|\\
& \leq & 2\|\cT_K\|^2 \|\cF_K -\cF_{K'}\| \|(K')^\top R K')\|
   +   \|\cT_K \|\|K^\top R K-(K')^\top R K')\|\\
& \leq & 2\|\cT_K\|^2 \|\cF_K -\cF_{K'}\| 
\left(\|(K')^\top R K') -K^\top R K\| +\|K^\top R K)\|\right)\\
&&   +   \|\cT_K \|\|K^\top R K-(K')^\top R K')\|\\
& \leq & \|\cT_K\| \|(K')^\top R K') -K^\top R K\| 
+2\|\cT_K\|^2 \|\cF_K -\cF_{K'}\| \|K^\top R K\|\\
&&   +   \|\cT_K \|\|K^\top R K-(K')^\top R K')\|\\
& = & 2\|\cT_K\| \|(K')^\top R K') -K^\top R K\| 
+2\|\cT_K\|^2 \|\cF_K -\cF_{K'}\| \|K^\top R K\|
\end{eqnarray*}
For the first term, 
\begin{align*}
2\|\cT_K\| \|(K')^\top R K') -K^\top R K\| 
&\leq 2\|\cT_K\| \left(2 \|K\|\|R\|\|K'-K\|+ \|R\| \|K'-K\|^2\right)\\
&\leq 2\|\cT_K\| \left(3 \|K\|\|R\|\|K'-K\|\right)
\end{align*}
using the assumption that $\|K'-K\| \leq \|K\|$.
For the second term,
\[
2\|\cT_K\|^2 \|\cF_K -\cF_{K'}\| \|K^\top R K\| \leq
2\|\cT_K\|^2 \, 2\|B\| \left( \|A-B K\|+ 1\right) \|K-K'\| \,
\|K\|^2\|R\| \, .
\]
Combining the two terms completes the proof.
\end{proof}

The next lemma shows the gradient is also stable after perturbation.

\begin{lemma}
($\nabla C_{K}$ perturbation) \label{lemma:gradperturbation}
Suppose $K'$ is such that:
\[
\|K'-K\|\leq \min \left(\frac {\sigma_{\min}(Q) \mu}
{4 C(K) \|B\|\left( \|A-B K\|+ 1\right) }, \|K\|\right)
\]
then there is a polynomial $\ph_{grad}$ in $\frac{C(K_0)}{\mu\sigma_{\textrm{min}}(Q)},\E[\|x_0\|^2],\|A\|,\|B\|,\|R\|,
\frac{1}{\sigma_{\textrm{min}}(R)}$ such that 
$$
\|\nabla C(K')-\nabla C(K)\|
\leq  \ph_{grad} \|K'-K\|.
$$
Also,
$$
\|\nabla C(K')-\nabla C(K)\|_F
\leq  \ph_{grad} \|K'-K\|_F.
$$
\end{lemma}

\begin{proof}
Recall $\nabla C(K) = 2E_K\Sigma_K$ where $E_K = (R+B^\top P_K B)K - B^\top P_K A$. Therefore
$$
\nabla C(K') - \nabla C(K) = 2E_{K'}\Sigma_{K'} - 2E_K\Sigma_K = 2(E_{K'} - E_K)\Sigma_{K'} + 2E_K(\Sigma_{K'} - \Sigma_K).
$$

Let's first look at the second term. By Lemma~\ref{lemma:domination},
\[
\Tr(E_K^\top E_K) \leq \frac{ \|R + B^\top P_KB\| (C(K)-C(K^*))}{\mu} ,
\]
then by Lemma~\ref{lemma:SigmaK_perturbation}
\[
\|\Sigma_{K'} - \Sigma_K \| \leq 
4 \left(\frac{C(K)}{\sigma_{\min}(Q) } \right)^2
\frac{ \|B\| \left(\|A-B K\|+ 1\right)}{ \mu} \|K-K'\| 
\]
Therefore the second term is bounded by
$$
8 \left(\frac{C(K)}{\sigma_{\min}(Q) } \right)^2
\frac{(\|R + B^\top P_KB\| (C(K)-C(K^*))) \|B\| \left(\|A-B K\|+ 1\right)}{ \mu^2} \|K-K'\|.
$$

Next we bound the first term. Since $K'-K$ is small enough, $\|\Sigma_{K'}\| \le \|\Sigma_K\|+\frac{C(K)}{\sigma_{min}(Q)}$.

For $E_{K'} - E_{K}$, we first need a  bound on $P_{K'} - P_K$.  By the previous lemma,
$$
\|P_K'-P_K\| = 6\left( 
 \left(\frac{C(K)}{\mu \,  \sigma_{\min}(Q)}\right)^2 \|K\|^2\|R\|\|B\| \left( \|A-B K\|+ 1\right)\
+ \left(\frac{C(K)}{\mu \,  \sigma_{\min}(Q)}\right) \|K\|\|R\|
\right) \|K-K'\|.
$$
Therefore
$$
E_K' - E_K = R(K'-K)+B^\top(P_{K'}-P_K) A + B^\top (P_{K'}-P_K)BK' + B^\top P_K B(K'-K).
$$
Since $\|K'\| \le 2\|K\|$, and $\|K\|$ can be bounded by $C(K)$
(Lemma~\ref{lemma:more_bounds}), all the terms can be bounded by
polynomials of related parameters multiplied by $\|K-K'\|$. 
%\rg{I could compute the bound but it would be impossible to parse...}
\end{proof}

\subsection{Smoothing and the gradient descent analysis}

This section analyzes the smoothing procedure and completes the proof
of gradient descent. Although Gaussian smoothing is more standard, the
objective $C(K)$ is not finite for every $K$, therefore technically
$\E_{u\sim \mathcal{N}(0,\sigma^2 I)}[C(K+u)]$ is not well
defined. This is avoidable by smoothing in a ball.

%For any radius $r$, define a probability distribution $\cP_r$ with density $p_r(u)$ that is proportional to $\|u\|^2 (r-\|u\|)^2$ for $\|u\| \le r$, and $0$ for all $\|u\| \ge r$. It is easy to verify that the density function has gradient in all of $\bR^d$. Define $C_r(K) = \E_{u\sim \cP_r}[C(K+u)]$.

Let $\mathbb{S}_{r}$ represent the uniform distribution over the
points with norm $r$ (boundary of a sphere), and $\mathbb{B}_r$
represent the uniform distribution over all points with norm at most
$r$ (the entire sphere). When applying these sets to matrix a $U$, the
Frobenius norm ball is used. The algorithm performs gradient descent
on the following function

$$
C_r(K) = \E_{U\sim \mathbb{B}_{r}}[C(K+U)].
$$

The next lemma uses the standard technique (e.g. in
\citep{flaxman2005online}) to show that the gradient of $C_r(K)$ can
be estimated just with an oracle for function value.

\begin{lemma}\label{lemma:stokes}
$\nabla C_r(K) = \frac{d}{r^2}\E_{U\sim \mathbb{S}_{r}}[C(K+U)U].$ 
\end{lemma}

This is the same as Lemma 2.1 in \citet{flaxman2005online}, for
completeness the proof is provided below.

\begin{proof}
By Stokes formula, 
$$
\nabla \int_{\delta \mathbb{B}_r} C(K+U) dx = \int_{\delta \mathcal{S}_r} C(K+U)\frac{U}{\|U\|_F} dx.
$$

By definition,
$$
C_r(K) = \frac{\int_{\delta \mathbb{B}_r} C(K+U) dx}{\mbox{vol}_d(\delta \mathbb{B}_r)},
$$
Also,
$$
\E_{U\sim \mathbb{S}_{r}}[C(K+U)U] = r\E_{U\sim \mathbb{S}_{r}}[C(K+U)\frac{U}{r}]= r\cdot \frac{\int_{\delta \mathcal{S}_r} C(K+U)\frac{U}{\|U\|_F} dx}{\mbox{vol}_{d-1}(\delta \mathbb{S}_r)}.
$$
The Lemma follows from combining these equations, and use the fact that
$$
{\mbox{vol}_d(\delta \mathbb{B}_r)} = \mbox{vol}_{d-1}(\delta \mathbb{S}_r)\cdot \frac{r}{d}.
$$

\end{proof}

From the lemma above and standard concentration inequalities, it is
immediate that it suffices to use a polynomial number of samples to approximate the gradient.

\begin{lemma}\label{lemma:approxgrad}
Given an $\epsilon$, there are fixed polynomials $\ph_r(1/\epsilon),\ph_{sample}(d,1/\epsilon)$ such that when $r \le 1/\ph_r(1/\epsilon)$, with $m \ge \ph_{sample}(d,1/\epsilon)$ samples of $U_1,...,U_n\sim \mathbb{S}_r$, with high probability (at least $1-(d/\epsilon)^{-d}$) the average
$$
\hat{\nabla} = \frac{1}{m}\sum_{i=1}^{m} \frac{d}{r^2}C(K+U_i)U_i
$$
is $\epsilon$ close to $\nabla C(K)$ in Frobenius norm.

Further, if for  $x\sim\cD$, $\|x\| \le L$ almost surely, there are
polynomials $\ph_{\it, grad}(d,1/\epsilon)$,
$\ph_{r,trunc}(1/\epsilon)$,
$\ph_{sample,trunc}(d,1/\epsilon,\sigma,L^2/\mu)$ such that when $m
\ge \ph_{sample,trunc}(d,1/\epsilon,L^2/\mu)$, $\ell \ge \ph_{\it,
  grad}(d,1/\epsilon)$, let $x^i_j, u^i_j (0\le j\le \it)$ be a single
path sampled using $K+U_i$,  then the average 
$$
\tilde{\nabla} = \frac{1}{m}\sum_{i=1}^{m} \frac{d}{r^2}[\sum_{j=0}^{\it-1}(x^i_j)^\top Qx^i_j+(u^i_j)^\top R u^i_j]U_i
$$
is also $\epsilon$ close to $\nabla C(K)$  in Frobenius norm with high probability.
\end{lemma}

%\rg{I just realized that the previous version requires taking expectation for starting points, I've modified everything to work in the case we sample one path for each perturbation $U_i$. The proof is a bit sketchy but what's missing are all standard concentration bounds...}

%\rg{Swapped sub-Gaussian to bounded}

\begin{proof}
For the first part, the difference is broken into two terms:
$$
\hat{\nabla} - \nabla C(K) = (\nabla C_r(K)-\nabla C(K))+(\hat{\nabla} - \nabla C_r(K)).
$$
For the first term, choose $\ph_r(1/\epsilon) = \min\{1/r_0, 2\ph_{grad}/\epsilon\}$ ($r_0$ is chosen later). By Lemma~\ref{lemma:gradperturbation} when $r$ is smaller than $1/\ph_r(1/\epsilon)=\epsilon/2\ph_{grad}$, every point $u$ on the sphere have $\|\nabla C(K+U)-\nabla C(K)\|_F \le \epsilon/4$. Since $\nabla C_r(K)$ is the expectation of $\nabla C(K+U)$, by triangle inequality $\|\nabla C_r(K)-\nabla C(K)\|_F \le \epsilon/2$.

The proof also makes sure that $r \le r_0$ such that for any $U \sim
\mathbb{S}_r$, it holds that $C(K+U) \le 2C(K)$. By
Lemma~\ref{lemma:objectiveperturbation}, $1/r_0$ is a
polynomial in the relevant factors. 

For the second term, by Lemma~\ref{lemma:stokes},
$\E[\hat{\nabla}] = \nabla C_r(K)$, and each individual sample has
norm bounded by $2dC(K)/r$, so by Vector Bernstein's Inequality,
know with $m \ge \ph_{sample}(d,1/\epsilon) =
\Theta\left(d\left(\frac{dC(K)}{\epsilon r}^2\right)\log
  {d/\epsilon}\right)$ samples, with high probability (at least
$1-(d/\epsilon)^{-d}$) $\|\hat{\nabla} - \E[\hat{\nabla}]\|_F \le
\epsilon/2$. 

%Finally for the second term, we observe that by the choice of $\delta$ the measure (with respect to $\cP_r$) of the range $\|u\|\le \delta r$ and $\|u\| \in [(1-\delta)r,r]$ is bounded by $\epsilon/2(\|\nabla C(K)\|+\epsilon)$. We also know that $\|\nabla C(K+u)\| \le (\|\nabla C(K)\|+\epsilon)$ for all $\|u\|\le r$. Therefore
%$$
%\|\nabla_{trunc} - \nabla C_r(K)\| = \|\E_{u\sim \cP_r}[(1-1_E(u))\nabla C(K+u)]\| \le \epsilon/2.
%$$

Adding these two terms and apply triangle inequality gives the result.
%Adding the three terms gives the result.

For the second part, the proof breaks it into more terms. Let $\nabla'$ be equal to $\frac{1}{m}\sum_{i=1}^{m} \frac{d}{r^2}C^{(\it)}(K+U_i)U_i$ (where $C^{(\it)}$ is defined as in Lemma~\ref{lemma:finitehorizon}), then 
$$
\tilde{\nabla} - \nabla C(K) = (\tilde{\nabla} - \nabla') + (\nabla' - \hat{\nabla}) + (\hat{\nabla}-\nabla C(K)).
$$

The third term is just what was bounded earlier, by choosing
$h_{r,trunc}(1/\epsilon) = h_r(2/\epsilon)$ and making sure
$h_{sample,trunc}(d,1/\epsilon) \ge h_{sample}(d,2/\epsilon)$, we
guarantees that it is smaller than $\epsilon/2$. 

For the second term, choose $\it \ge \frac{16d^2\cdot
  C^2(K)(\|Q\|+\|R\|\|K\|^2)}{\epsilon r\mu \sigma_{min}^2(Q)} =:
\ph_{\it, grad}(d,1/\epsilon)$. By Lemma~\ref{lemma:finitehorizon},
for any $K'$ with $C(K') \le 2C(K)$, it holds that
$\|C^{(\it)}(K')-C(K')\| \le \frac{r\epsilon}{4d}$. Therefore by
triangle inequality 
$$
\|\frac{1}{m}\sum_{i=1}^{m} \frac{d}{r^2}C^{(\it)}(K+U_i)U_i - \frac{1}{m}\sum_{i=1}^{m} \frac{d}{r^2}C(K+U_i)U_i\| \le \epsilon/4.
$$

Finally for the first term it is easy to see that $\E[\tilde{\nabla}]
= \nabla'$ where the expectation is taken over the randomness of the
initial states $x^i_0$. 
Since $\|x^i_0\| \le L$, $(x^i_0)(x^i_0)^\top \preceq
\frac{L^2}{\mu} \E[x_0x_0^\top]$, as a result the sum 
$$
[\sum_{j=0}^{\it-1}(x^i_j)^\top Qx^i_j+(u^i_j)^\top R u^i_j] \le \frac{L^2}{\mu} C(K+U_i).
$$
Therefore, $\tilde{\nabla} - \nabla'$ is again a sum of independent
vectors with bounded norm, so by Vector Bernstein's inequality, when $\ph_{sample,trunc}(d,1/\epsilon,L^2/\mu)$ is a large enough
polynomial, $\|\tilde{\nabla} - \nabla'\| \le \epsilon/4$ with high
probability. Adding all the terms finishes the proof.  
\end{proof}
%\rg{There were some typos above but they are now fixed.}

\begin{theorem}\label{theorem:gdperturb}
There are fixed polynomials $\ph_{GD, r}(1/\epsilon), \ph_{GD,
  sample}(d,1/\epsilon,L^2/\mu),\ph_{GD,\it}(d,1/\epsilon)$ such that
if every step the gradient is computed as Lemma~\ref{lemma:approxgrad}
(truncated at step $\it$), pick step size $\eta$ and $T$ the same as
the gradient descent case of Theorem~\ref{theorem:gd_exact}, it holds that $C(K_T) - C(K^\star) \le \epsilon$ with high probability (at least $1-\exp(-d)$).
\end{theorem}

\begin{proof}
%By the proof of Theorem~\ref{theorem:gd}, we know before $C(K_T)-C(K^\star) \le \epsilon$, the function value decreases by a fixed polynomial at every step. Combined with Lemma~\ref{lemma:objectiveperturbation}, this means if we compute $K_{t+1}$ with some fixed polynomial accuracy, then the amount of decrease at every step can be maintained (up to a factor of 1/2). To compute $K_{t+1}$ accurately, we rely on Lemma~\ref{lemma:approxgrad}, which states a fixed polynomial $r$ and sample size is enough. Finally we apply a union bound to show all the steps are successful with high probability.

By Lemma~\ref{lemma:gd}, when $\eta \le 1/\ph_{GD,\eta}$ for some
fixed polynomial $\ph_{GD, \eta}$(given in Lemma~\ref{lemma:gd}), then 
\[
C(K') -C(K^*) \leq \left(1-\eta \sigma_{\textrm{min}}(R) \frac{\mu^2}{\|\Sigma_{K^*}\|}\right) (C(K)-C(K^*))
\]

Let $\tilde{\nabla}$ be the approximate gradient computed, and let
$K'' = K - \eta \tilde{\nabla}$ be the iterate that uses the
approximate gradient. The proof shows given enough samples, the
gradient can be estimated with enough accuracy that makes sure  
$$
|C(K'') - C(K')| \le \frac{1}{2} \eta \sigma_{\textrm{min}}(R) \frac{\mu^2}{\|\Sigma_{K^*}\|}\cdot \epsilon.
$$
This means as long as $C(K)-C(K^*) \ge \epsilon$, it holds that
$$
C(K'') - C(K^*) \leq \left(1-\frac{1}{2}\eta \sigma_{\textrm{min}}(R) \frac{\mu^2}{\|\Sigma_{K^*}\|}\right) (C(K)-C(K^*)).
$$
Then the same proof of Theorem~\ref{theorem:gd_exact} gives the convergence guarantee.

Now $C(K'') - C(K')$ is bounded. By
Lemma~\ref{lemma:objectiveperturbation}, if $\|K''-K'\| \le
\frac{1}{2} \eta \sigma_{\textrm{min}}(R)
\frac{\mu^2}{\|\Sigma_{K^*}\|}\cdot \epsilon \cdot 1/\ph_{func}$
($\ph_{func}$ is the polynomial in
Lemma~\ref{lemma:objectiveperturbation}), then $C(K'') - C(K')$ is
small enough. To get that, observe $K'' - K' = \eta (\nabla -
\tilde{\nabla})$, therefore it suffices to make sure 
$$
\|\nabla - \tilde{\nabla}\| \le \frac{1}{2}  \sigma_{\textrm{min}}(R) \frac{\mu^2}{\|\Sigma_{K^*}\|}\cdot \epsilon \cdot 1/\ph_{func}
$$
By Lemma~\ref{lemma:gradperturbation}, it suffices to pick
$\ph_{GD,r}(1/\epsilon) =
\ph_{r,trunc}(2\ph_{func}\|\Sigma_{K^*}\|/(\mu^2\sigma_{min}(R)\epsilon))$,
$\ph_{GD, sample}(d, 1/\epsilon,L^2/\mu) = \ph_{sample,trunc}(d,
2\ph_{func}\|\Sigma_{K^*}\|/(\mu^2\sigma_{min}(R)\epsilon),L^2/\mu)$,
and $\ph_{GD,\it}(d,1/\epsilon) = \ph_{\it, grad}(d,
2\ph_{func}\|\Sigma_{K^*}\|/(\mu^2\sigma_{min}(R)\epsilon))$. This
gives the desired upper-bound on $\|\nabla - \tilde{\nabla}\|$ with
high probability (at least $1-(\epsilon/d)^{-d}$). 

Since the number of steps is a polynomial, by union bound with high
probability (at least $1-T(\epsilon/d)^{-d} \ge 1-\exp(-d)$) the
gradient is accurate enough for all the steps, so 

$$
C(K'') - C(K^*) \leq \left(1-\frac{1}{2}\eta \sigma_{\textrm{min}}(R) \frac{\mu^2}{\|\Sigma_{K^*}\|}\right) (C(K)-C(K^*)).
$$

The rest of the proof is the same as
Theorem~\ref{theorem:gd_exact}. Note that in the smoothing, because
the function value is monotonically decreasing and the choice of
radius, all the function value encountered is bounded by $2C(K_0)$,
so the polynomials are indeed bounded throughout the algorithm. 
\end{proof}

\subsection{The natural gradient analysis}

Before the Theorem for natural gradient is proven, the following lemma
shows the variance can be estimated accurately.

\begin{lemma}\label{lemma:approxvariance}
If for  $x\sim\cD$, $\|x\| \le L$ almost surely, there exists
polynomials $\ph_{r, var}(1/\epsilon)$, $\ph_{varsample,
  trunc}(d,1/\epsilon,L^2/\mu)$ and $\ph_{\it, var}(d,1/\epsilon)$
such that if $\hat{\Sigma}_K$ is estimated using at least $m \ge
\ph_{varsample, trunc}(d,1/\epsilon, L^2/\mu)$ initial points $x^1_0,
..., x^m_0$, $m$ random perturbations $U_i \sim \mathbb{S}_r$ where $r
\le 1/\ph_{r,var}(1/\epsilon)$, all of these initial points are
simulated using $\hat{K}_i = K+U_i$ to $\it \ge \ph_{\it,
  var}(d,1/\epsilon)$ iterations, then with high probability (at least
$1-(d/\epsilon)^{-d}$) the following estimate 
%$$
%\|\hat{\Sigma}_K - \Sigma_K\| \le \epsilon,
%$$

%Further,  there are polynomials $\ph_{\it, var}(d,1/\epsilon)$ and $\ph_{varsample,trunc}(d,1/\epsilon,L^2/\mu)$ such that when $n \ge \ph_{varsample,trunc}(d,1/\epsilon,L^2/\mu)$, $\ell \ge \ph_{\it, grad}(d,1/\epsilon)$, let $x^i_j, u^i_j (0\le j\le \it)$ be a single path sampled using $K+U_i$,  we have that the average
$$
\tilde{\Sigma} = \frac{1}{m}\sum_{i=1}^{m} \sum_{j=0}^{\it-1}x^i_j(x^i_j)^\top.
$$
satisfies $\|\tilde{\Sigma} - \Sigma_K\| \le \epsilon$.
Further, when $\epsilon \le \mu/2$, it holds that $\sigma_{min}(\hat{\Sigma}_K) \ge \mu/2$. 
\end{lemma}
%\rg{This is in a similar situation as Lemma 24.}
\begin{proof}
%\rg{I was somehow completely confused... we don't need to do smoothing for $\Sigma_K$...}
%Note that $\Sigma_K$ is just the variance of a Gaussian random variable, and it is bounded by $\|\Sigma_K\| \le \frac{C(K)}{\sigma_{min}(Q)}$ (Lemma~\ref{lemma:bounds}). Let $\{x_0,...\}$ be a sample of the trajectory and $v(u) = \sum_{i=0}^\infty x_i x_i^\top$, the variance $\|v(x)\|^2$ is the 4-th moment of a Gaussian and is bounded by $3\|\Sigma_K\|_F^2$. The variable $v(x)$ also has a subexponential tail. Therefore by standard concentration inequalities we have when
%$$
%n \ge \Gamma \frac{dC(K)^2}{\sigma_{min}^2(Q)\epsilon^2} \log (d/\epsilon),
%$$
%where $\Gamma$ is a large enough constant, $\|\hat{\Sigma}_K - \Sigma_K\| \le \epsilon/2$ with probability at least $1-(d/\epsilon)^{-d}$.

This is broken into three terms: let $\Sigma^{(\it)}_K$ be defined as
in Lemma~\ref{lemma:finitehorizon}, let $\hat{\Sigma} =
\frac{1}{m}\sum_{i=1}^m \Sigma_{K+U_i}$ and $\hat{\Sigma}^{(\it)} =
\frac{1}{m}\sum_{i=1}^m \Sigma^{(\it)}_{K+U_i}$, then  it holds that
$$
\tilde{\Sigma} - \Sigma_K = (\tilde{\Sigma} - \hat{\Sigma}^{(\it)}) + (\hat{\Sigma}^{(\it)} - \hat{\Sigma}) + (\hat{\Sigma} - \Sigma_K).
$$

First, $r$ is chosen small enough so that $C(K+U_i) \le 2C(K)$. This
only requires an inverse polynomial $r$ by
Lemma~\ref{lemma:objectiveperturbation}. 

For the first term, note that $\E[\tilde{\Sigma}] =
\hat{\Sigma}^{(\it)}$ where the expectation is taken over the initial
points $x^i_0$. Since $\|x^i_0\| \le L$, $(x^i_0)(x^i_0)^\top \preceq
\frac{L^2}{\mu} \E[x_0x_0^\top]$, and as a result the sum
$$
\sum_{j=0}^{\it-1}x^i_j(x^i_j)^\top Q \preceq \frac{L^2}{\mu}\Sigma_{K+U_i}.
$$
Therefore, standard concentration bounds show that when
$\ph_{varsample,trunc}$ is a large enough polynomial,
$\|\tilde{\Sigma} - \hat{\Sigma}^{(\it)}\| \le \epsilon/2$ holds with high
probability. 

For the second term, Lemma~\ref{lemma:finitehorizon} is applied. Because $C(K+U_i) \le 2C(K)$, choosing $\it \ge \ph_{\it, var}(d,1/\epsilon) = \frac{8d\cdot C^2(K)}{\epsilon\mu \sigma_{min}^2(Q)}$, the error introduced by truncation $\|\hat{\Sigma}^{(\it)} - \hat{\Sigma}\|$ is then bounded by $\epsilon/4$. 

For the third term, Lemma~\ref{lemma:SigmaK_perturbation} is applied. When $r \le \epsilon\cdot \left(\frac{\sigma_{\min}(Q) }{C(K)} \right)^2
\frac{\mu}{16 \|B\| \left(\|A-B K\|+ 1\right)}$, 
$\|\Sigma_{K+U_i} - \Sigma_K\|\le \epsilon/4$. Since $\hat{\Sigma}$ is
the average of $\Sigma_{K+U_i}$, by the triangle inequality, 
$\|\hat\Sigma - \Sigma_K\|\le \epsilon/4$. 

  Adding these three terms gives the result.

Finally, the bound on $\sigma_{min}(\tilde{\Sigma}_K)$ follows simply from Weyl's Theorem.
\end{proof}

\begin{theorem}\label{theorem:ngdperturb}
Suppose $C(K_0)$ is
finite and and $\mu>0$. The natural gradient follows the update rule:
\[
K_{t+1} = K_t - \eta \nabla C(K_t) \Sigma_{K_{t}}^{-1}
\]
Suppose the stepsize is set to be:
\[
\eta = \frac{1}{\|R\| + \frac{\|B\|^2 C(K_0)}{\mu}}
\]
If the gradient and variance are estimated as in
Lemma~\ref{lemma:approxgrad}, Lemma~\ref{lemma:approxvariance} with $r
= 1/\ph_{NGD, r}(1/\epsilon)$, with $m \ge \ph_{NGD, sample}(d,
1/\epsilon, L^2/\mu)$ samples, both are truncated to
$\ph_{NGD,\it}(d,1/\epsilon)$ iterations, then with high probability
(at least $1-\exp(-d)$) in $T$ iterations where 
\[
T>\frac{\|\Sigma_{K^*}\|}{\mu} \,
\left(\frac{\|R\|}{\sigma_{\textrm{min}}(R)} + 
\frac{\|B\|^2 C(K_0)}{\mu \sigma_{\textrm{min}}(R)} \right) 
\, \log \frac{2(C(K_0) -C(K^*))}{\eps}
\]
then the natural gradient satisfies the following performance bound:
\[
C(K_T) -C(K^*) \leq \eps
\]
\end{theorem}

\begin{proof}
By Lemma~\ref{lemma:ngd}, 
\[
C(K') -C(K^*) \leq \left(1-\eta \sigma_{\textrm{min}}(R) \frac{\mu}{\|\Sigma_{K^*}\|}\right) (C(K)-C(K^*))
\]

Let $\tilde{\nabla}$ be the estimated gradient, $\tilde{\Sigma}_K$ be
the estimated $\Sigma_K$, and let $K'' = K - \eta
\tilde{\nabla}\tilde{\Sigma_K}^{-1}$. The proof shows that when both the
gradient and the covariance matrix are estimated accurately enough, then
$$
|C(K')-C(K'')| \le \frac{\epsilon}{2}\eta \sigma_{\textrm{min}}(R) \frac{\mu}{\|\Sigma_{K^*}\|}.
$$
This implies when $C(K)-C(K^\star) \ge \epsilon$, 
\[
C(K') -C(K^*) \leq \left(1-\frac{1}{2}\eta \sigma_{\textrm{min}}(R) \frac{\mu}{\|\Sigma_{K^*}\|}\right) (C(K)-C(K^*))
\]
which is sufficient for the proof.

By Lemma~\ref{lemma:objectiveperturbation}, if $\|K'' - K'\|
\le \frac{\epsilon}{2\ph_{func}}\eta \sigma_{\textrm{min}}(R)
\frac{\mu}{\|\Sigma_{K^*}\|}$ the desired bound on
$|C(K')-C(K'')|$ holds. To achieve this, it suffices to have
$$
\|\tilde{\nabla}\tilde{\Sigma}_K^{-1} - \nabla C(K)\Sigma_K^{-1}\| \le \frac{\epsilon}{2\ph_{func}} \sigma_{\textrm{min}}(R) \frac{\mu}{\|\Sigma_{K^*}\|}.
$$

This is broken into two terms
$$
\|\tilde{\nabla}\tilde{\Sigma}_K^{-1} - \nabla C(K)\Sigma_K^{-1}\| \le \|\tilde{\nabla}-\nabla\|\|\tilde{\Sigma}_K^{-1}\| + \|\nabla C(K)\| \|\tilde{\Sigma}_K^{-1} - \Sigma_K^{-1}\|.$$

For the first term, by Lemma~\ref{lemma:approxvariance} we know when
the number of samples is large enough $\|\tilde{\Sigma}_K^{-1}\| \le
2/\mu$. Therefore it suffices to make sure $\|\tilde{\nabla}-\nabla\|
\le \frac{\epsilon }{8\ph_{func}} \sigma_{\textrm{min}}(R)
\frac{\mu^2}{\|\Sigma_{K^*}\|}$, this can be done by
Lemma~\ref{lemma:approxgrad} by setting $\ph_{NGD,grad,r}(1/\epsilon)
=
\ph_{r,trunc}(\frac{8\ph_{func}\|\Sigma_K^*\|}{\mu^2\sigma_{min}(R)\epsilon})$,\\
$\ph_{NGD, gradsample}(d,1/\epsilon,L/\mu^2) = \ph_{sample,trunc}(d,
\frac{8\ph_{func}\|\Sigma_K^*\|}{\mu^2\sigma_{min}(R)\epsilon},L/\mu^2)$
and \\$\ph_{NGD, \it, grad}(d,1/\epsilon) = \ph_{\it, grad}(d,
\frac{8\ph_{func}\|\Sigma_K^*\|}{\mu^2\sigma_{min}(R)\epsilon})$. 

For the second term, it suffices to make sure $\|\tilde{\Sigma}_K^{-1}
- \Sigma_K^{-1}\| \le \frac{\epsilon}{4\ph_{func}}
\sigma_{\textrm{min}}(R) \frac{\mu}{\|\Sigma_{K^*}\|\|\nabla
  C(K)\|}.$. By standard matrix perturbation, if 
$\sigma_{min}(\Sigma_K) \ge \mu$ and $\|\tilde{\Sigma}_K - \Sigma_K\|\le \mu/2$, $\|\tilde{\Sigma}_K^{-1} -
\Sigma_K^{-1}\| \le 2\|\tilde{\Sigma}_K - \Sigma_K\|/\mu^2$. Therefore by
Lemma~\ref{lemma:approxvariance} it suffices to choose 
$\ph_{NGD, var, r}(1/\epsilon) = \ph_{var,
  r}(\frac{8\ph_{func}\|\Sigma_{K^*}\|\|\nabla
  C(K)\|}{\mu^3\sigma_{min}(R)\epsilon})$,  
$\ph_{NGD, varsample}(d,1/\epsilon,L/\mu^2) = \ph_{varsample,trunc}(d,\frac{8\ph_{func}\|\Sigma_{K^*}\|\|\nabla C(K)\|}{\mu^3\sigma_{min}(R)\epsilon},L/\mu^2)$ and $\ph_{NGD, \it, var}(d,1/\epsilon) = \ph_{\it, var}(d,\frac{8\ph_{func}\|\Sigma_{K^*}\|\|\nabla C(K)\|}{\mu^3\sigma_{min}(R)\epsilon})$.
This is indeed a polynomial because $\|\nabla C(K)\|$ is bounded by
Lemma~\ref{lemma:more_bounds}. 

Finally, choose $\ph_{NGD,r} = \max\{\ph_{NGD,grad, r}, \ph_{NGD,var, r}\}$, \\$\ph_{NGD,sample} = \max\{\ph_{NGD,gradsample}, \ph_{NGD,varsample}\}$, and $\ph_{NGD,\it} = \max\{\ph_{NGD,\it, grad},\ph_{NGD,\it, var}\}$. This ensures all the bounds mentioned above hold and that
\[
C(K') -C(K^*) \leq \left(1-\frac{1}{2}\eta \sigma_{\textrm{min}}(R) \frac{\mu}{\|\Sigma_{K^*}\|}\right) (C(K)-C(K^*))
\]
The rest of the proof is the same as
Theorem~\ref{theorem:gd_exact}. Note again that in the smoothing,
because the function value is monotonically decreasing and the choice
of radius, all the function values encountered are bounded by
$2C(K_0)$, so the polynomials are indeed bounded throughout the
algorithm. 
\end{proof}

\subsection{Standard Matrix Perturbation and Concentrations}

In the previous sections, we used several standard tools in matrix perturbation and concentration, which we summarize here. The matrix perturbation theorems can be found in \citet{stewart1990matrix}. Matrix concentration bounds can be found in \citet{tropp2012user}

\begin{theorem}[Weyl's Theorem] Suppose $B = A+E$, then the singular values of $B$ are within $\|E\|$ to the corresponding singular values of $A$. In particular $\|B\| \le \|A\|+\|E|$ and $\sigma_{min}(B) \ge \sigma_{min}(A) - \|E\|$. 
\end{theorem}

\begin{theorem}[Perturbation of Inverse] Let $B = A+E$, suppose $\|E\| \le \sigma_{min}(A)/2$ then $\|B^{-1}-A^{-1}\| \le 2\|A-B\|/\sigma_{min}(A)$.
\end{theorem}

\begin{theorem}[Matrix Bernstein]
Suppose $\hat{A} = \sum_i \hat{A}_i$, where $\hat{A}_i$ are independent random matrices of dimension $d_1\times d_2$ (let $d = d_1+d_2$). Let  $\E[\hat{A}] = A$, the variance $M_1 = \E[\sum_i \hat{A}_i \hat{A}_i^\top]$, $M_2 = \E[\sum_i \hat{A}_i^\top \hat{A}_i]$. If $\sigma^2 = \max\{\|M_1\|, \|M_2\|\}$, and every $\hat{A}_i$ has spectral norm $\|\hat{A}_i\| \le R$ with probability 1, then with high probability
$$
\| \hat{A} - A\| \le O(R\log d + \sqrt{\sigma^2 \log d}).
$$
\end{theorem}

In our proof we often treat a matrix as a vector and look at its Frobenius norm, in these cases we use the following corollary:

\begin{theorem}[Vector Bernstein]
Suppose $\hat{a} = \sum_i \hat{a}_i$, where $\hat{a}_i$ are independent random vector of dimension $d$. Let  $\E[\hat{a}] = a$, the variance $\sigma^2 = \E[\sum_i \|\hat{a}_i\|^2]$. If every $\hat{a}_i$ has norm $\|\hat{a}_i\| \le R$ with probability 1, then with high probability
$$
\| \hat{a} - a\| \le O(R\log d + \sqrt{\sigma^2 \log d}).
$$
\end{theorem}

\section{Simulation Results}
\label{sec:exp}
Here we give simulations for the gradient descent algorithm (with backtracking step size) to show that the algorithm indeed converges within reasonable time in practice. In this experiment, $x \in \mathbb{R}^{100}$ and $u \in \mathbb{R}^{20}$. We use random matrices $A, B$. The scaling of $A$ is chosen so that $A$ is stabilizing with high probability ($\lambda_{max}(A) \le 1$). We initialize the solution at $K_0 = 0$, which ensures $C(K_0)$ is finite because $A$ is stabilizing. The distribution of the initial point $x_0$ is the unit cube. We computed the gradient using Lemma~\ref{lem:gradexpression}. See Figure~\ref{fig:exp} for the result. Although the example uses the exact gradient for the each iterative step, it provides a glimpse into the potential use 
of first order methods and their variants for direct feedback gain update in LQR.

%
%Convergence of the gradient descent algorithm (with backtracking step size) on a particular realization of LQR; in this example,
%$x \in \reals^{100}$ and $u \in \reals^{20}$. The gradient update on the state feedback gain is initialized from the zero gain
%as the open loop dynamics has been constructed to be stable (with a random normally distributed initial condition on the unit cube for the state).
%Although the example uses the exact gradient for the each iterative step, it provides a glimpse into the potential use 
%of first order methods and their variants for direct feedback gain update in LQR."

\begin{figure}
\centering
\includegraphics[height=3in]{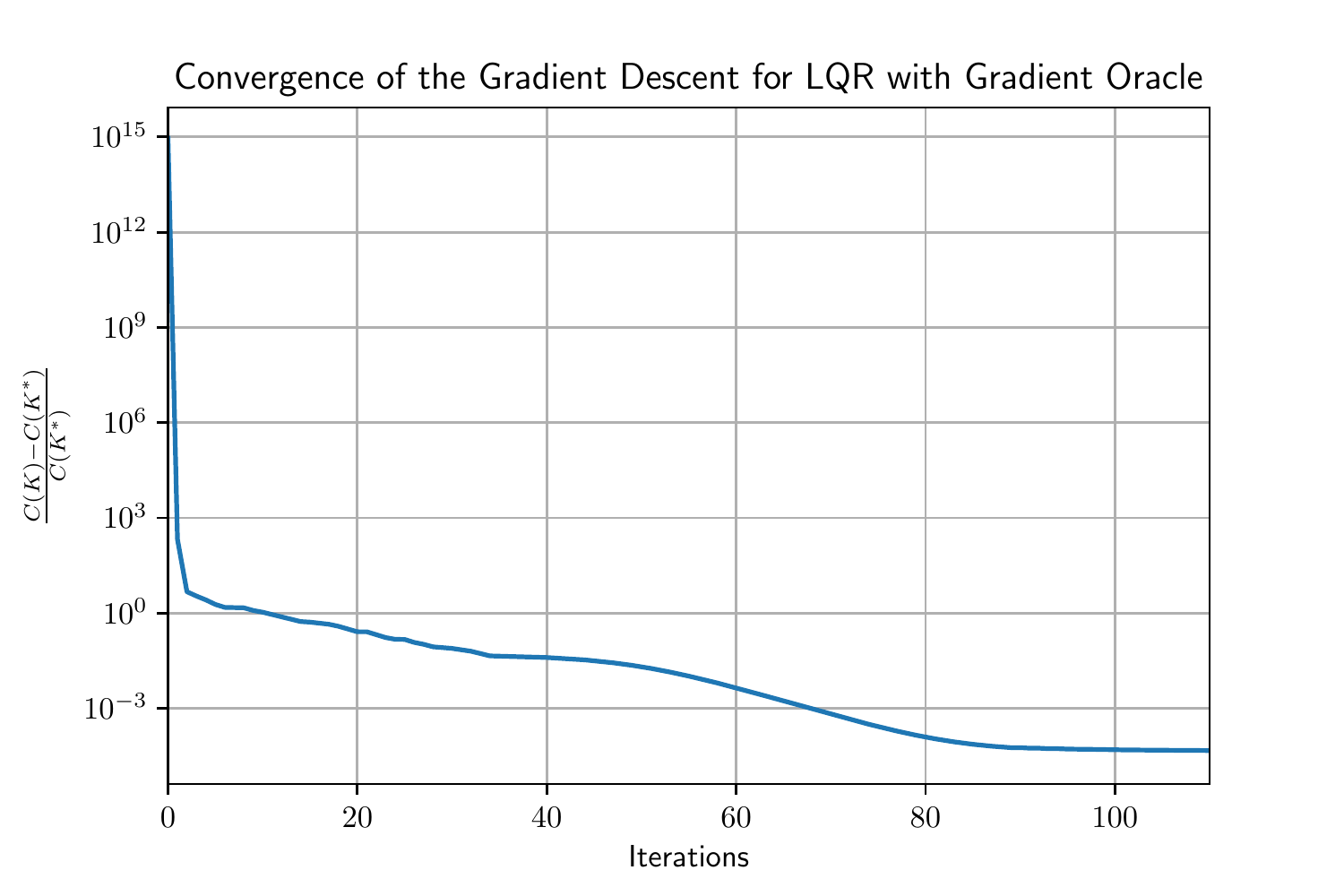}
\caption{Simulation results with Gradient Descent
\textsuperscript{*}}\small \textsuperscript{*} The simulation was done by Jingjing Bu.
\label{fig:exp}
\end{figure}

\end{document}